\documentclass{article}

\usepackage{arxiv}
\usepackage{natbib}

\usepackage{amsmath,amssymb,amsthm}

\usepackage{booktabs,graphicx,subcaption}

\usepackage{xcolor}
\usepackage{microtype}
\usepackage{enumitem}

\usepackage{tikz}
\usetikzlibrary{positioning,arrows.meta,calc,
                decorations.pathmorphing,shapes.geometric}

\usepackage{hyperref}
\definecolor{darkblue}{RGB}{0,60,120}
\hypersetup{colorlinks=true,linkcolor=darkblue,
            citecolor=darkblue,urlcolor=darkblue}

\theoremstyle{plain}
\newtheorem{theorem}{Theorem}
\newtheorem{proposition}[theorem]{Proposition}

\theoremstyle{definition}

\theoremstyle{remark}
\newtheorem{remark}[theorem]{Remark}

\newcommand{\tguomargin}[1]{}

\newcommand{\norm}[1]{\left\|#1\right\|}
\newcommand{\dosc}{{d_{\mathrm{osc}}}}

\newcommand{\kstar}{z_i^*}
\newcommand{\Wij}{w_{ij}}
\newcommand{\qj}{r_j}
\newcommand{\aij}{a_{ij}}

\renewcommand{\undertitle}{Published in Transactions on Machine Learning Research (2026)}
\renewcommand{\headeright}{TMLR 2026}

\title{Attention by Synchronization in Coupled Oscillator Networks%
  \thanks{OpenReview forum: \url{https://openreview.net/forum?id=T9g7VleXf2}}}

\author{%
  Fabio Pasqualetti \\
  Electrical Engineering and Computer Science\\
  University of California, Irvine\\
  \texttt{fabiopas@uci.edu}
  \and
  Taosha Guo \\
  Electrical Engineering and Computer Science\\
  University of California, Irvine\\
  \texttt{taoshag@uci.edu}
}

\usepackage{todonotes}

\begin{document}
\maketitle

\begin{abstract}
  We address transformer attention on energy-constrained physical substrates.
  Softmax attention requires exponentiation and global reduction, operations
  with high energy cost on von Neumann hardware and no natural physical analog.
  We show that Kuramoto synchronization dynamics (which arise in electrical,
  mechanical, superconducting, and charge-density-wave oscillator arrays, among
  other physical systems) implement a well-defined attention operation. The
  resulting mechanism, \emph{fixed-query oscillator attention}, replaces
  softmax's arithmetic with the equilibration of a gradient flow on the sphere:
  queries are learned anchors fixed on the sphere, and free oscillators evolve
  under Kuramoto--Lohe dynamics until they settle at positions encoding
  attention weights via cosine similarity. Because the computation is
  equilibration, no global exponential normalization is needed. The fixed point
  is provably unique and globally attractive from almost every initial
  condition, a guarantee that holds across every physical
  realization. Empirically, at the minimal hardware configuration (oscillator
  dimension $\dosc{=}2$), oscillator attention matches softmax on keyword
  spotting, and on subject-verb agreement it trains more reliably while reaching
  softmax's accuracy. Softmax retains an advantage on causal language modeling,
  but the gap decays as a power law in $\dosc$.  The main objective of this work
  is not to replace softmax in software but to provide a mathematically grounded
  blueprint for accurate attention on physical substrates.
\end{abstract}

\section{Introduction}
The attention mechanism is the computational core of modern
transformers~\citep{AV-NS-NP-JU-LJ-AG-LK-IP:17}, but its energy cost is
unavoidable on von Neumann hardware. Computing all-pairs query-key similarities
and applying a global exponential normalization scales quadratically with
sequence length, dominating inference cost on always-on edge devices such as
wearables, embedded sensors, and autonomous systems. The sustained memory
traffic exceeds the energy budgets of energy-harvested
hardware~\citep{YT-MD-DB-DM:22}, so the limit on transformer-grade inference
at the edge is power, not model quality.

This paper does not aim to replace softmax in software, where it is
well-optimized, understood, and hard to improve upon. Our goal is to
design attention mechanisms that physical systems can implement natively,
without approximating a digital algorithm in analog hardware. This requires
starting from what physics does naturally and asking whether that computation
can serve as attention. Attention is fundamentally a \emph{consensus
operation}: each token settles to a distribution over its neighbors reflecting
pairwise similarity, amplifying close neighbors and attenuating distant ones.
Consensus is also what physical oscillator networks, as well as other
multi-agent systems, compute when they equilibrate. The
Kuramoto model, introduced in 1975 to describe synchronization in populations of
coupled nonlinear oscillators~\citep{YK:75}, has since been shown to govern the
dynamics of remarkably diverse physical systems: electrical circuits, mechanical
pendulums, superconducting Josephson junction arrays~\citep{KW-PC-SHS:98}, and
charge-density-wave oscillator arrays \citep{JOB-TG-FP-AAB:25}. In each, coupled
units find a shared phase configuration that balances their pairwise
interactions. The energy cost of this computation is determined by the
substrate, not by floating-point arithmetic. If attention is consensus, and
physical oscillator networks compute consensus, then physical oscillator
networks can, in principle, compute attention.

We develop this \emph{physical intelligence} perspective (the view that a useful
computation can be a property of a dynamical class rather than a specific
device) by introducing \emph{fixed-query oscillator attention}, grounded in the
Lohe model~\citep{MAL:09}, a high-dimensional generalization of the Kuramoto
equations for oscillators on a unit sphere. The mechanism uses two kinds of
oscillators in different roles, generalizing the query-key structure of softmax
attention: anchors are fixed reference points on the sphere (the queries),
learned during training, while free oscillators evolve under input-dependent
coupling weights derived from key-like projections, settling toward whichever
anchors their couplings favor. At equilibrium, cosine similarities between
settled oscillators and anchors yield row-stochastic attention weights via
affine normalization, without a global exponential normalization. The
analysis is substrate-independent: wherever Kuramoto dynamics arise naturally,
they can compute attention. This universality extends to neuroscience:
Kuramoto dynamics are a canonical model of neural
oscillations~\citep{MB-SH-AD:10}, and the fixed-query structure (reference
phases anchoring the network while input-driven oscillators synchronize to them)
resembles theoretical accounts of how cortical areas use sustained oscillatory
templates for selective attention~\citep{WS-CMG:95, AKE-PF-WS:01}.  We do not
claim a direct biological model, but if neural oscillations compute something
like attention, the mechanism described here may point toward biologically
plausible sequence modeling~architectures.

This paper makes five contributions. First, we introduce fixed-query oscillator
attention as a blueprint for physically realizable attention: we establish its
closed-form fixed point for efficient training, prove that every design choice
is dictated by a physical constraint of the substrate, and identify the
dynamical class (coupled oscillators with attractive interactions on the sphere)
whose physical realizations can implement it. Second, we prove that when the
weighted anchor sum is nonzero, the dynamics have exactly two equilibria (one
stable and one unstable), and every trajectory except those starting at the
unstable equilibrium converges to the stable one. We characterize the
probability of slow convergence for trajectories starting near the unstable
equilibrium, and show that with positive coupling weights, the weighted anchor
sum vanishes with probability that decays exponentially in the oscillator
dimension. Third, we demonstrate that the physical mechanism is a viable
attention substrate: at the minimal hardware configuration (oscillator
  dimension $\dosc{=}2$, one oscillator per token), oscillator attention matches
  softmax on keyword spotting under matched positional information, and on
  subject-verb agreement at a constrained-capacity configuration with 9
    free oscillators it fails to train significantly less often ($3/50$ seeds
  versus $12/50$, $p = 0.023$). On causal language modeling the gap
  closes monotonically as the oscillator dimension grows from $2$ to $32$,
  consistent with a power law, and the fitted relation predicts a held-out
  $\dosc{=}64$ configuration on both datasets. Fourth, ablations with a
  frozen random value projection show that a learned value projection is
  unnecessary on both bidirectional tasks: on keyword spotting the frozen-$W_V$
  oscillator stays within $0.5$ pp of the full model, and on subject-verb
  agreement freezing $W_V$ leaves accuracy statistically unchanged
    ($-0.80$~pp, not significant). These ablations, however, do not by
  themselves isolate the dynamics as the source of accuracy, since the remaining
  pathways stay trainable. Fifth, we propose that the language modeling gap
arises from a dimensional bottleneck: the oscillator attention pattern is
constrained to a low-dimensional manifold whose size grows with the oscillator
dimension, while softmax operates with the full feature dimension throughout. At
low oscillator dimension, this constraint discards directions of similarity that
softmax can use; higher oscillator dimension progressively recovers them. We
support this view with the observed scaling behavior, with the gap following an
approximate $\dosc^{-1/2}$ law in the oscillator dimension, providing a
practical oscillator-budget design rule for target tasks. Separately, we show
that readout sharpening provides a software-side optimization that improves
accuracy at fixed $\dosc$ (up to $+1.23$~pp on KWS, up to $0.80$~PPL on causal
language modeling).

\section{Fixed-query oscillator attention}\label{sec:mechanism}
We introduce fixed-query oscillator attention as a drop-in replacement for the
softmax attention module in a standard transformer. The mechanism is
substrate-independent: the same equations describe its execution on a mechanical
oscillator network, an electrical oscillator circuit, a Josephson junction
array, or, for training, in software via the closed-form fixed point. The
starting point is the Lohe model~\citep{MAL:09}, which describes the dynamics of
$n$ oscillators $x_1,\ldots,x_n$ on the unit sphere
$\mathbb{S}^{\dosc-1} \subset \mathbb{R}^{\dosc}$ evolving under pairwise
coupling:
\begin{equation}\label{eq:lohe}
  \dot x_i = \Omega_i x_i + \left(I - x_i x_i^\top\right)
  \sum_{j=1}^n w_{ij}\, x_j,
\end{equation}
where $\dosc \geq 2$ is the oscillator dimension,
$\Omega_i \in \mathbb{R}^{\dosc \times \dosc}$ is a skew-symmetric matrix
encoding the natural frequency of oscillator $i$, $w_{ij} \geq 0$ is the
coupling weight between oscillators $i$ and $j$, and the projection
$(I - x_i x_i^\top)$ ensures that the trajectories evolve on the sphere. When
$\dosc{=}2$, writing $x_i = (\cos\theta_i, \sin\theta_i)$,
Equation~\eqref{eq:lohe} reduces to the scalar Kuramoto
model~\citep{YK:75,SHS:00}
$\dot\theta_i = \omega_i + \sum_j w_{ij}\sin(\theta_j - \theta_i)$, a
canonical model of coupled electrical and mechanical oscillators
(Appendix~\ref{app:kuramoto-derivation} shows this derivation). We set
$\Omega_i = 0$ throughout; non-zero $\Omega_i$ could encode positional
information and are a direction for future work.


Fixed-query oscillator attention specializes~\eqref{eq:lohe} to the transformer
setting. Let $e_1,\ldots,e_T \in \mathbb{R}^{d_{\rm model}}$ be input token
embeddings with $T$ the sequence length. The mechanism partitions the
oscillators into two asymmetric roles, with one of each kind per input token.
\emph{Anchor oscillators} $r_j \in \mathbb{S}^{\dosc-1}$, $j=1,\ldots,T$, are
reference points on the sphere. They are learned during training and held fixed
at inference: they do not evolve during the dynamics. The anchors are
  parameterized per absolute sequence position and per head: $r_j$ is a learned
  parameter indexed by the position $j$, shared across inputs and independent of
  token content. The anchors therefore carry positional information, a point we
  return to and quantify in Section~\ref{sec:bidirectional}. The anchors play
the role that learned queries play in softmax attention, but they are now
positions on the sphere rather than free vectors. \emph{Free oscillators}
$z_i \in \mathbb{S}^{\dosc-1}$, $i=1,\ldots,T$, are the dynamical
variables. They evolve under the influence of the anchors via the Lohe
equation~\eqref{eq:lohe}, pulled by positive, input-dependent coupling~weights
\begin{equation}
  \Wij = \sigma\!\left(\frac{(Fe_i)^\top (Ge_j)}{\sqrt{d_h}}\right),
  \label{eq:coupling}
\end{equation}
where $F, G \in \mathbb{R}^{d_h \times d_{\rm model}}$ are learned projection
matrices and $\sigma : \mathbb{R} \to \mathbb{R}_{>0}$ is any strictly positive
function (we use softplus, $\sigma(x) = \log(1+e^x)$, an element-wise positive
nonlinearity that avoids the global reduction required by
softmax). Softplus is a convenience although it is not required: a smooth
  exponential coupling is natural (e.g., strictly positive, and it sharpens
  large similarities) but the ablations show it is not required. Replacing it
  with $\mathrm{relu}(x)+10^{-3}$ or $\mathrm{elu}(x)+1$ leaves accuracy
  essentially unchanged on both bidirectional and language-modeling tasks
  (Appendix~\ref{app:details}). In the analogy with softmax attention, the
anchors $r_j$ play the role of queries (fixed reference points), the free
oscillators $z_i$ play the role of keys (input-dependent states settled by the
dynamics), and $F, G$ are projection matrices analogous to $W_Q, W_K$ in
standard attention; the value projection $W_V$ acts on the value aggregation as
usual. The free oscillators follow the Lohe dynamics~\eqref{eq:lohe}:
\begin{equation}
  \dot z_i = \left(I - z_i z_i^\top\right)
  \underbrace{\sum_{j=1}^T \Wij\, r_j}_{\displaystyle h_i},
  \label{eq:fixedquery}
\end{equation}
where $h_i = \sum_j \Wij r_j \in \mathbb{R}^{\dosc}$ is the \emph{weighted
  anchor sum}. Since $z_i$ is constrained to the unit sphere, it settles at the
unit vector in the direction of $h_i$. The dynamics~\eqref{eq:fixedquery}
therefore converge to the point on $\mathbb{S}^{\dosc-1}$ most aligned with
$h_i$ (Section~\ref{sec:theory}, Theorem~\ref{thm:uniqueness}, proves uniqueness
and global stability):
\begin{equation}
  \kstar = \frac{h_i}{\norm{h_i}}.
  \label{eq:fixedpoint}
\end{equation}
Cosine similarities between settled free oscillators and anchors
are read out as normalized attention weights:
\begin{equation}\label{eq:readout}
  \aij = \frac{1 + \kstar{}^\top r_j}
              {\displaystyle\sum_{l=1}^{T}
               \left(1 + \kstar{}^\top r_l\right)},
\end{equation}
which is a linear normalization of shifted cosine similarities; the only sum is
the affine denominator $\sum_l (1 + z_i^{*\top} r_l)$, with no
exponentiation.\footnote{A power $p{\geq}1$ in the numerator and denominator is
  a natural sharpening generalization that improves performance over $p{=}1$ on
  the tasks we evaluate (Section~\ref{sec:readout}). We use $p{=}1$ throughout
  the headline experiments because it can be the hardware-native readout;
  $p > 1$ is available as a software-side improvement when digital
  post-processing per oscillator is acceptable.} A note on terminology:
  synchronization here means phase-locking of the free oscillators to the fixed
  anchors, the regime of pacemaker-driven Kuramoto networks and of injection
  locking. The free oscillators do not interact with one another. Such coupling
  and nonzero natural frequencies, which are deliberately excluded here and left
  for future work, is what yields the closed form and the convergence
  guarantee. For autoregressive applications, attention is restricted to past
tokens by setting $\Wij = 0$ for $j > i$ in~\eqref{eq:coupling} and restricting
the sum in~\eqref{eq:readout} to $l \leq i$; this is standard causal masking and
is orthogonal to the mechanism. The module operates with $H$ independent
attention heads of dimension $d_h = d_{\rm model}/H$; we have described one
head, indexing all per-head quantities by $(h)$ below. The output of head $h$ is
the weighted combination of value vectors $v_j^{(h)} = W_V^{(h)} e_j$:
\begin{equation}
  o_i^{(h)} = \sum_{j=1}^{T} \aij^{(h)}\, v_j^{(h)}.
  \label{eq:output}
\end{equation}
Each head maintains its own $r_j^{(h)}$, $F^{(h)}, G^{(h)}, W_V^{(h)}$,
producing output $o_i^{(h)}$ as in~\eqref{eq:output}. The $H$ head outputs are
projected by $W_O \in \mathbb{R}^{d_{\rm model} \times d_{\rm model}}$,
following the standard transformer block~\citep{AV-NS-NP-JU-LJ-AG-LK-IP:17}.

Training and inference share a single fixed point. During training, the
closed-form fixed point~\eqref{eq:fixedpoint} is used directly: it is
differentiable in all parameters and costs $O(T \cdot \dosc)$ per token (a
weighted sum over $T$ anchors in $\mathbb{R}^{\dosc}$, then a normalization).
At hardware inference, the oscillator array runs Equation~\eqref{eq:fixedquery}
physically until it settles. Both paths converge to the same $\kstar$, so there
is no algorithmic gap between training and hardware inference, only a
finite-time approximation gap that Section~\ref{sec:convergence} characterizes
empirically. Figure~\ref{fig:architecture} shows the resulting hybrid pipeline:
coupling weights and anchor positions are computed digitally in the front-end,
the oscillator network equilibrates the dynamics~\eqref{eq:fixedquery}, and the
digital back-end reads out attention weights via~\eqref{eq:readout} and feeds
them to the remaining transformer~block. In a purely digital setting there may
be no reason to prefer oscillator attention over softmax; the latter is
simpler, faster, and better understood. The oscillator mechanism is designed for
physical substrates, and we evaluate it against softmax solely as a calibration
of how much the physically constrained mechanism costs in accuracy relative to
the unconstrained digital baseline. This cost is modest: on bidirectional tasks
oscillator attention matches or exceeds softmax, and on causal language modeling
the gap closes predictably with the oscillator dimension.


\begin{figure}[t]
\centering
\begin{tikzpicture}[
  font=\small, >=Stealth,
  digital/.style={rectangle, rounded corners=3pt,
    draw=blue!55!black, line width=0.8pt, fill=blue!8,
    minimum width=3.2cm, minimum height=2.4cm,
    align=center, inner sep=5pt},
  analog/.style={rectangle, rounded corners=4pt,
    draw=brown!70!black, line width=1.1pt, fill=orange!9,
    minimum width=4.6cm, minimum height=2.4cm,
    align=center, inner sep=5pt},
  keyosc/.style={circle, draw=blue!80!black, line width=1.0pt,
    fill=blue!35, minimum size=5pt, inner sep=0pt},
  queryosc/.style={rectangle, draw=red!80!black, line width=1.0pt,
    fill=red!25, minimum size=5pt, inner sep=0pt},
  spring/.style={decorate,
    decoration={coil, aspect=0.5, segment length=3pt,
      amplitude=2pt, pre length=2pt, post length=2pt},
    draw=brown!65!black, line width=0.6pt},
  arrowlabel/.style={font=\scriptsize\itshape, text=black!70},
]
\node[digital] (frontend) at (0,0)
  {\textbf{Digital front-end}\\[3pt]
   \scriptsize\itshape token embeddings\\
   \scriptsize\itshape couplings $\Wij$};

\node[analog, right=1.6cm of frontend] (analog)
  {\phantom{X}\\[28pt]\phantom{X}};
\node[anchor=north, font=\small\bfseries\itshape,
      text=brown!70!black, inner sep=0pt]
  at ([yshift=-4pt]analog.north) {Oscillator network};

\coordinate (ctr) at ([yshift=-6pt]analog.center);
\draw[brown!40!black, line width=0.8pt] (ctr) circle (.75cm);
\foreach \j/\ang in {0/90,1/180,2/270,3/0}
  \node[queryosc] (q\j) at ($(ctr)+(\ang:.75cm)$) {};
\foreach \i/\ang in {0/45,1/135,2/225,3/315}
  \node[keyosc]   (k\i) at ($(ctr)+(\ang:.75cm)$) {};
\draw[spring](k0)--(q0); \draw[spring](k0)--(q3);
\draw[spring](k1)--(q0); \draw[spring](k1)--(q1);
\draw[spring](k2)--(q1); \draw[spring](k2)--(q2);
\draw[spring](k3)--(q2); \draw[spring](k3)--(q3);

\node[digital, right=1.6cm of analog] (backend)
  {\textbf{Digital back-end}\\[3pt]
   \scriptsize\itshape readout~\eqref{eq:readout}\\
   \scriptsize\itshape aggregation~\eqref{eq:output}};

\draw[->, line width=0.85pt] (frontend.east) -- (analog.west)
  node[midway, above=2pt, arrowlabel] {$\Wij,\;\qj$};
\draw[->, line width=0.85pt] (analog.east) -- (backend.west)
  node[midway, above=2pt, arrowlabel] {$\kstar$};

\node[anchor=north, font=\scriptsize, inner sep=0pt]
  at ([yshift=-6pt]analog.south){%
  \textcolor{red!70!black}{$\blacksquare$}\ anchor $r_j$
  (fixed at inference)\quad
  \textcolor{blue!70!black}{$\bullet$}\ free oscillator $z_i$};
\end{tikzpicture}
\caption{\textbf{Fixed-query oscillator attention.} Coupling weights $\Wij$ and
  anchor positions $r_j$ are computed digitally and loaded into the oscillator
  array. Free oscillators $z_i$ evolve on the sphere
  under~\eqref{eq:fixedquery}, pulled toward the fixed anchors $r_j$ by springs
  encoding $\Wij$. The settled positions $\kstar$ are read out as attention
  weights~\eqref{eq:readout} by the digital back-end. The equilibration runs in
  physical dynamics, not in von Neumann arithmetic; the energy savings come from
  this stage, while the digital front-end and back-end retain conventional
  costs.}
\label{fig:architecture}
\end{figure}

\begin{remark}[\emph{Substrate-dictated design choices}]
\label{rem:hardware}
Every architectural choice in the mechanism is constrained by what physical
oscillator substrates can implement, rather than by software optimization.
\emph{Positive coupling ($\Wij > 0$)}: passive coupling elements (resistive,
capacitive, inductive) are positive-valued; any coupling function
$\sigma : \mathbb{R} \to \mathbb{R}_{>0}$ ensures this, and
Theorem~\ref{thm:uniqueness} shows that this guarantees a unique fixed point.
\emph{Sphere geometry ($z_i, r_j \in \mathbb{S}^{\dosc-1}$)}: physical
oscillators have a natural state space defined by their dynamics:
$\mathbb{S}^1$ for phase oscillators, $\mathbb{S}^{\dosc-1}$ for
higher-dimensional substrates.  \emph{Fixed query anchors}: driving oscillators
with a fixed reference signal (injection locking in electrical
systems~\citep{ATS-SC-CD-MA-TG-EC-SK-JN-MJ-MJA-BLB:22}, entrainment in
mechanical and biological ones) is the standard way to steer oscillator
networks toward chosen targets, and it implements $r_j$ in hardware.
\emph{Cosine readout, linear normalization}: inner products on the sphere are
the natural observable from oscillator hardware, and the hardware measures
$\kstar{}^\top r_j$. The linear normalization in the readout requires only
division, which is cheap to perform digitally on the back-end; the
exponentiation in softmax has no efficient physical analog, and is the
operation we replace.  \hfill $\square$
\end{remark}

\begin{remark}[\emph{Why dynamics rather than direct
    normalization}]\label{rem:whydynamics}
  Attention weights could be computed directly from their algebraic
    expression and without any dynamics. However, using the dynamics for the
    computation of attention improves energy efficiency, as argued, and
    robustness.
    In fact, computing the formula directly means every component error reaches
    the output. Reaching the same weights as the equilibrium of a dynamical
    system means errors are absorbed as the state relaxes, because that
    equilibrium is attracting. Theorem~\ref{thm:uniqueness} proves that the
    equilibrium is unique and reached from almost every starting point, and
    Appendix~\ref{app:robustness} shows that task accuracy unchanged when the
    programmed couplings are perturbed by tens of percent.  \hfill $\square$
\end{remark}

\section{Theoretical analysis}\label{sec:theory}
For a generic Kuramoto or Lohe network, multiple stable equilibria typically
coexist, and which one is reached depends sensitively on initial conditions. The
fixed-query design eliminates this ambiguity: because anchors $r_j$ are external
forcing terms rather than free oscillators, each free oscillator $z_i$ minimizes
a potential with exactly one minimum on $\mathbb{S}^{\dosc-1}$.
Theorem~\ref{thm:uniqueness} establishes uniqueness and global stability, and
Propositions~\ref{prop:degenerate} and~\ref{prop:antipodal} characterize the two
ways finite-time convergence can fail in practice, both controlled by the
oscillator dimension $\dosc$. Throughout, we treat each free oscillator $z_i$
independently: under the fixed-query dynamics~\eqref{eq:fixedquery}, the
evolution of $z_i$ depends only on the anchors $r_j$ and couplings $w_{ij}$ (not
on the other free oscillators) and is governed by the gradient flow induced by
the weighted anchor sum $h_i = \sum_j w_{ij}\, r_j$.

The results below follow directly from the fixed-query structure, and
  their value lies in what they guarantee for hardware: uniqueness and
  almost-global attractivity hold for every physical realization of the flow,
  which justifies training on the algebraic fixed point and deploying the
  physics.

\begin{theorem}[Uniqueness and global stability]\label{thm:uniqueness}
  Given anchors $r_1, \ldots, r_T \in \mathbb{S}^{\dosc-1}$ and weights
  $w_{i1}, \ldots, w_{iT}$, let $h_i = \sum_{j=1}^T w_{ij}\, r_j$ and assume
  that $\|h_i\| > 0$. Then the gradient flow
  \begin{equation}\label{eq:gradflow}
    \dot z_i = (I - z_i z_i^\top)\, h_i,
    \qquad z_i \in \mathbb{S}^{\dosc-1},
  \end{equation}
  has exactly two equilibria, $z_i^* = h_i/\norm{h_i}$ and $-z_i^*$. The
  equilibrium $z_i^*$ is asymptotically stable with basin of attraction
  $\mathbb{S}^{\dosc-1} \setminus \{-z_i^*\}$; the equilibrium $-z_i^*$ is
  unstable.
\end{theorem}

\begin{proof}
  Equilibria of~\eqref{eq:gradflow} satisfy $(I - z_i z_i^\top)\, h_i = 0$.
  Combined with $\norm{z_i}=1$, this yields exactly two solutions
  $z_i = \pm z_i^*$. Consider the energy function
  $V(z_i) = -z_i^\top h_i$. Differentiating along the trajectories
  of~\eqref{eq:gradflow},
  \[
    \dot V
    = -h_i^\top \dot z_i
    = -h_i^\top (I - z_i z_i^\top)\, h_i
    = -\bigl(\norm{h_i}^2 - (z_i^\top h_i)^2\bigr)
    \leq 0,
  \]
  where the last inequality follows from the Cauchy--Schwarz inequality applied
  to $z_i^\top h_i$ with $\norm{z_i}=1$. Equality holds if and only if
  $(z_i^\top h_i)^2 = \norm{h_i}^2$ or, equivalently, if and only if
  $z_i = \pm z_i^*$. By LaSalle's invariance principle on the compact,
  positively invariant manifold $\mathbb{S}^{\dosc-1}$, every trajectory
  converges to the set $\{z_i^*, -z_i^*\}$~\citep{HKK:02-bis}. For any initial
  condition $z_i(0) \neq -z_i^*$, the Cauchy--Schwarz inequality gives
  $V(z_i(0)) < V(-z_i^*) = \norm{h_i}$. Since $V$ is non-increasing along the
  trajectories of~\eqref{eq:gradflow}, $V(z_i(t)) < \norm{h_i}$ for all
  $t \geq 0$, ruling out convergence to $-z_i^*$. Hence every trajectory with
  $z_i(0) \neq -z_i^*$ converges asymptotically to $z_i^*$.
\end{proof}

Theorem~\ref{thm:uniqueness} guarantees asymptotic convergence under the
hypothesis $\norm{h_i} > 0$, but practical inference stops the dynamics after a
finite integration window, and the hypothesis itself can be violated when the
anchor sum collapses. Two distinct failure modes can therefore compromise the
practical implementation of the proposed scheme: (i) \emph{degenerate anchor
  positions}, where $\norm{h_i}$ is anomalously small and the gradient
flow~\eqref{eq:gradflow} has a vanishing driving force; and (ii) \emph{antipodal
  initialization}, where the trajectory starts close to the unstable equilibrium
$-z_i^*$. The softplus activation in~\eqref{eq:coupling} keeps all
couplings $\Wij$ strictly positive, ensuring that $h_i$ is a positive
combination of the anchors; this does not by itself rule out
$\norm{h_i}=0$ (e.g., on $\mathbb{S}^1$ with $r_2 = -r_1$ and $w_1 = w_2$),
but such exact cancellations are a measure-zero event for generic
anchor configurations. More importantly, $\norm{h_i}$ can still be small
even when nonzero, slowing the dynamics. The two propositions below
quantify both failure modes and show their probabilities decay exponentially
with $\dosc$.


\begin{proposition}[Degenerate positions]\label{prop:degenerate}
  Let anchors $r_1, \ldots, r_T$ be drawn independently and uniformly from
  $\mathbb{S}^{\dosc-1}$,
  $w_i = [w_{i1}, \ldots, w_{iT}]^\top \in \mathbb{R}^T_{>0}$ be a vector of
  positive weights, and $h_i = \sum_{j=1}^T w_{ij}\, r_j$. Then
  \[
    \mathbb{E}\bigl[\norm{h_i}^2\bigr] = \norm{w_i}^2,
  \]
  and for every $\varepsilon \in (0,1)$ and for some absolute constant $c > 0$
  \[
    \Pr\bigl(\norm{h_i}^2 \leq \varepsilon\, \norm{w_i}^2\bigr) \leq
    2\exp\!\bigl(-c\,(\dosc-1)(1-\varepsilon)^2\bigr).
  \]
\end{proposition}

\begin{proof}
  For $r_j$ drawn uniformly from $\mathbb{S}^{\dosc-1}$,
  \[
    \mathbb{E}[r_j] = 0,
    \qquad
    \mathbb{E}[\norm{r_j}^2] = 1.
  \]
  The first follows from antipodal symmetry of the sphere, the second
  from $\norm{r_j} = 1$. Expanding $\norm{h_i}^2 = h_i^\top h_i$ as a double
  sum over independent indices $j, k \in \{1,\ldots,T\}$,
  \[
    \mathbb{E}[\norm{h_i}^2]
    = \mathbb{E}\Big[\sum_{j,k} w_{ij} w_{ik}\, r_j^\top r_k\Big]
    = \sum_j w_{ij}^2\, \mathbb{E}[\norm{r_j}^2]
    + \sum_{j \neq k} w_{ij} w_{ik}\, \mathbb{E}[r_j^\top r_k]
    = \norm{w_i}^2,
  \]
  where the second equality uses linearity of expectation and isolates the
  diagonal and off-diagonal terms, and the third uses
  $\mathbb{E}[\|r_j\|^2] = 1$ together with independence and zero mean of
  the $r_j$
  ($\mathbb{E}[r_j^\top r_k] = \mathbb{E}[r_j]^\top \mathbb{E}[r_k] = 0$).

  Stack the anchors into the column vector
  $\bar r = [r_1^\top, \ldots, r_T^\top]^\top \in \mathbb{R}^{T\dosc}$,
  noting that $\mathbb{E}[\bar r] = 0$. L\'evy's concentration of measure
  on $\mathbb{S}^{\dosc-1}$~\citep[Thm.~5.7]{boucheron2013concentration}
  states that for every $1$-Lipschitz function $f : \mathbb{R}^{\dosc}
  \to \mathbb{R}$,
  \[
    \Pr\bigl(\bigl|f(r_j) - \mathbb{E}[f(r_j)]\bigr| \geq t\bigr)
    \leq 2\exp\bigl(-c_0(\dosc-1)t^2\bigr)
  \]
  for an absolute constant $c_0 > 0$. Since this bound holds for all
  $1$-Lipschitz functions, it holds in particular for all $1$-Lipschitz convex
  functions. Thus, following~\citep[Def.~2.2]{adamczak2015note}, the random
  vector $r_j$ satisfies the convex concentration property with constant
  $K = 1/\sqrt{c_0(\dosc-1)}$. Since the $r_j$ are independent, the joint vector
  $\bar r$ inherits the convex concentration property with the same constant
  $K$~\citep[Chap.~6]{boucheron2013concentration}. By Theorem~2.5
  of~\citep{adamczak2015note}, there exists an absolute constant $c_1 > 0$ such
  that, for every symmetric matrix $A \in \mathbb{R}^{T\dosc \times T\dosc}$ and
  every $t > 0$,
  \begin{align}\label{eq:adamczak}
    \Pr\bigl(\bigl|\bar r^\top A \bar r - \mathbb{E}[\bar r^\top A \bar r]\bigr|
            \geq t\bigr)
    \leq 2\exp\!\biggl(
    -c_1\min\!\Bigl(
      \tfrac{t^2}{2K^4\,\norm{A}_{\rm F}^2},
      \tfrac{t}{K^2\,\norm{A}_{2}}
    \Bigr)
    \biggr).
  \end{align}
  We apply this to $A = w_i w_i^\top \otimes I_{\dosc} \in
  \mathbb{R}^{T\dosc \times T\dosc}$, the symmetric block matrix with
  $\dosc \times \dosc$ blocks $[A]_{jk} = w_{ij} w_{ik}\, I_{\dosc}$ for
  $j, k \in \{1, \ldots, T\}$. By block matrix multiplication,
  \[
    \bar r^\top A \bar r
    = \sum_{j,k} r_j^\top [A]_{jk}\, r_k
    = \sum_{j,k} w_{ij} w_{ik}\, r_j^\top r_k
    = \norm{h_i}^2.
  \]
  Each block $[A]_{jk} = w_{ij} w_{ik}\, I_{\dosc}$ has $\dosc$ nonzero entries,
  each equal to $w_{ij} w_{ik}$, so
  $\norm{[A]_{jk}}_{\rm F}^2 = \dosc\,(w_{ij} w_{ik})^2$
  and, summing,
  \[
    \norm{A}_{\rm F}^2
    = \sum_{j,k} \dosc\,(w_{ij} w_{ik})^2
    = \dosc\,\Big(\sum_j w_{ij}^2\Big)\Big(\sum_k w_{ik}^2\Big)
    = \dosc\,\norm{w_i}^4.
  \]
  For the spectral norm, $A = w_i w_i^\top \otimes I_{\dosc}$ is a Kronecker
  product. Its eigenvalues are pairwise products of eigenvalues of
  $w_i w_i^\top$ and $I_{\dosc}$. The matrix $w_i w_i^\top$ has eigenvalue
  $\norm{w_i}^2$ once and zero $T-1$ times, while $I_{\dosc}$ has eigenvalue $1$
  with multiplicity $\dosc$. Therefore the nonzero eigenvalues of $A$ are
  $\norm{w_i}^2$ with multiplicity $\dosc$ and, because $A$ is symmetric, the
  spectral norm equals the largest eigenvalue: $\norm{A}_{2} = \norm{w_i}^2$.

  Substituting $K^2 = 1/(c_0(\dosc-1))$,
  $\norm{A}_{\rm F}^2 = \dosc\,\norm{w_i}^4$, $\norm{A}_{2} = \norm{w_i}^2$,
  and $\mathbb{E}[\bar r^\top A \bar r] = \norm{w_i}^2$ into~\eqref{eq:adamczak}:
  \[
    \frac{t^2}{2K^4\,\norm{A}_{\rm F}^2}
      = \frac{c_0^2\,(\dosc-1)^2\,t^2}{2\,\dosc\,\norm{w_i}^4},
    \qquad
    \frac{t}{K^2\,\norm{A}_{2}}
      = \frac{c_0\,(\dosc-1)\,t}{\norm{w_i}^2}.
  \]
  Hence
  \begin{align}\label{eq:prob_inequality}
    \Pr\bigl(\bigl|\norm{h_i}^2 - \norm{w_i}^2\bigr| \geq t\bigr)
    \leq 2\exp\!\biggl(
    -c_2\min\!\Bigl(
      \tfrac{(\dosc-1)^2\,t^2}{\dosc\,\norm{w_i}^4},
      \tfrac{(\dosc-1)\,t}{\norm{w_i}^2}
    \Bigr)
    \biggr),
  \end{align}
  where $c_2 > 0$ absorbs $c_0$, $c_0^2$, $c_1$, and the factor $1/2$.

  To conclude, fix $\varepsilon \in (0,1)$ and set
  $t = (1-\varepsilon)\,\norm{w_i}^2$. The event
  $\norm{h_i}^2 \leq \varepsilon\,\norm{w_i}^2$ is equivalent to
  $\norm{w_i}^2 - \norm{h_i}^2 \geq (1-\varepsilon)\,\norm{w_i}^2$, hence is
  contained in $\bigl\{\bigl|\norm{h_i}^2 - \norm{w_i}^2\bigr| \geq t\bigr\}$.
  Substituting into~\eqref{eq:prob_inequality},
  \[
    \Pr\bigl(\norm{h_i}^2 \leq \varepsilon\,\norm{w_i}^2\bigr)
    \leq 2\exp\!\biggl(
    -c_2\min\!\Bigl(
      \tfrac{(\dosc-1)^2\,(1-\varepsilon)^2}{\dosc},
      (\dosc-1)(1-\varepsilon)
    \Bigr)
    \biggr).
  \]
  For $\dosc \geq 2$ and $\varepsilon \in (0,1)$, the bounds
  $(\dosc-1)/\dosc \geq 1/2$ and $1-\varepsilon \geq (1-\varepsilon)^2$ give
  \[
    \min\!\Bigl(
    \tfrac{(\dosc-1)^2\,(1-\varepsilon)^2}{\dosc},
    (\dosc-1)(1-\varepsilon)
    \Bigr)
    \geq \tfrac{1}{2}\,(\dosc-1)(1-\varepsilon)^2,
  \]
  yielding
  \[
    \Pr\bigl(\norm{h_i}^2 \leq \varepsilon\,\norm{w_i}^2\bigr)
    \leq 2\exp\!\bigl(-c\,(\dosc-1)(1-\varepsilon)^2\bigr)
  \]
  with $c = c_2/2$.
\end{proof}

The proposition shows that as $\dosc$ grows, $\norm{h_i}$ concentrates sharply
around its mean $\norm{w_i}$, so degenerate positions
($\norm{h_i} \ll \norm{w_i}$) become exponentially rare. In autoregressive
settings where oscillator $i$ has only a subset of active anchors (those
$j \leq i$ under causal masking), the bound applies with
$\norm{w_i}^2 = \sum_j w_{ij}^2$ counting only those active anchors. Early
positions therefore have a small absolute scale $\norm{w_i}$ (few active
anchors), but the concentration is in $\dosc$ rather than in the number of
active anchors, so the relative magnitude $\norm{h_i}/\norm{w_i}$ that governs
the gradient flow is controlled regardless of position; increasing the number of
anchors does not improve this, while increasing their dimension does. This bound
assumes uniformly distributed anchors and is therefore an initialization-time
guarantee. On trained checkpoints the assumption no longer holds, yet the
  conclusion still does: across all three tasks, the measured fraction of
  oscillators with $\norm{h_i} < 0.01$ ranges from none on SVA to
    $0.43\%$ on KWS ($0.025\%$ on the no-positional-encoding checkpoint), so
  the degenerate set remains negligible after training
  (Appendix~\ref{app:robustness}).

  The second failure mode is independent of $\norm{h_i}$ and instead concerns
  the initial condition. Theorem~\ref{thm:uniqueness} guarantees convergence to
  $z_i^*$ for every $z_i(0) \neq -z_i^*$, but trajectories initialized close to
  the unstable equilibrium $-z_i^*$ leave its neighborhood only slowly, and a
  finite integration window may end before the trajectory has reached a useful
  neighborhood of $z_i^*$. To quantify how often this happens, we model an
  initialization in which the free oscillator is independent of the dynamics,
  with $z_i(0)$ drawn uniformly from $\mathbb{S}^{\dosc-1}$. This
    corresponds to a cold start (e.g., a random phase set by injection-lock
    noise) that carries no information about the dynamics, and any warm start
    would concentrate away from $-z_i^*$.

%

\begin{proposition}[Antipodal initialization]\label{prop:antipodal}
  Let $z_i(0)$ be drawn uniformly from $\mathbb{S}^{\dosc-1}$, independently of
  the anchors and weights, and let $z_i^* \in \mathbb{S}^{\dosc-1}$ be the
  stable equilibrium of~\eqref{eq:gradflow}. For every $\alpha \in (0,\pi)$,
  \begin{align}\label{eq:capprob}
    \Pr\bigl(\angle(z_i(0),-z_i^*) < \alpha\bigr)
    = \frac{1}{\sqrt{\pi}}\,
    \frac{\Gamma\!\bigl(\tfrac{\dosc}{2}\bigr)}
    {\Gamma\!\bigl(\tfrac{\dosc-1}{2}\bigr)}
    \int_0^\alpha \sin^{\dosc-2}(\theta)\,d\theta.
  \end{align}
  For every $\alpha \in (0, \pi/2)$, the probability in~\eqref{eq:capprob}
  decays exponentially in $\dosc$.  Moreover, for each fixed
  $\dosc \geq 2$,
  \[
    \Pr\bigl(\angle(z_i(0),-z_i^*) < \alpha\bigr)
    \sim C(\dosc)\,\alpha^{\dosc-1}
    \qquad \text{as } \alpha \to 0,
  \]
  for a constant $C(\dosc) > 0$.
\end{proposition}
\begin{proof}
  The probability that a uniformly distributed point on $\mathbb{S}^{\dosc-1}$
  falls within angular distance $\alpha$ of a fixed point equals the ratio of
  the spherical cap area to the total sphere surface area; see, e.g., Lecture~8
  of~\citep{ball1997elementary}. Slicing $\mathbb{S}^{\dosc-1}$ by hyperplanes
  perpendicular to the axis through the cap's pole, the slice at polar angle
  $\theta$ is a $(\dosc-2)$-sphere of radius $\sin\theta$. Integrating the slice
  areas over $\theta \in [0, \alpha]$ gives the cap area
  \[
    \mathrm{Area}\bigl(\mathbb{S}^{\dosc-2}\bigr)\,
    \int_0^\alpha \sin^{\dosc-2}(\theta)\,d\theta,
  \]
  where
  $\mathrm{Area}\bigl(\mathbb{S}^{\dosc-2}\bigr) = 2\pi^{(\dosc-1)/2} /
  \Gamma\bigl((\dosc-1)/2\bigr)$ (where $\Gamma$ is the Euler Gamma function)
  is the surface area of the
  $(\dosc-2)$-dimensional unit sphere. The total surface area of
  $\mathbb{S}^{\dosc-1}$ is $2\pi^{\dosc/2}/\Gamma(\dosc/2)$. Taking the ratio,
  we recover \eqref{eq:capprob}:
  \begin{align*}
    \Pr\bigl(\angle(z_i(0),-z_i^*) < \alpha\bigr) &= \frac{\text{cap
        area}}{\text{total area}} =
    \frac{2\pi^{(\dosc-1)/2}/\Gamma((\dosc-1)/2)}
    {2\pi^{\dosc/2}/\Gamma(\dosc/2)} \cdot \int_0^\alpha
    \sin^{\dosc-2}(\theta)\,d\theta  \\ &= \frac{1}{\sqrt{\pi}}\,
    \frac{\Gamma\!\bigl(\tfrac{\dosc}{2}\bigr)}
    {\Gamma\!\bigl(\tfrac{\dosc-1}{2}\bigr)} \int_0^\alpha
    \sin^{\dosc-2}(\theta)\,d\theta.
  \end{align*}

  For any $\alpha \in (0,\pi/2)$ and for all $\theta \in [0,\alpha]$, we have
  $\sin\theta \leq \sin\alpha < 1$. Thus,
  \[
    \int_0^\alpha \sin^{\dosc-2}(\theta)\,d\theta \leq
    \alpha\,\sin^{\dosc-2}(\alpha),
  \]
  which decays exponentially in $\dosc$. Stirling's
  approximation~\citep[6.1.47]{abramowitz1964handbook} gives
  $\Gamma(\dosc/2)/\Gamma((\dosc-1)/2) \sim \sqrt{\dosc/2}$ as
  $\dosc \to \infty$, so the prefactor in~\eqref{eq:capprob} grows only as
  $\sqrt{\dosc}$. Multiplying by the exponential decay of the integral, the
  probability decays exponentially in $\dosc$ for every fixed
  $\alpha \in (0, \pi/2)$. For small values of $\alpha$, the expansion
  $\sin\theta = \theta + O(\theta^3)$ gives
  \[
    \int_0^\alpha \sin^{\dosc-2}(\theta)\,d\theta
    = \frac{\alpha^{\dosc-1}}{\dosc-1} + O\bigl(\alpha^{\dosc+1}\bigr).
  \]
  The remaining factors, the prefactor $\Gamma(\dosc/2)/(\sqrt{\pi}\,
  \Gamma((\dosc-1)/2))$ and the $1/(\dosc-1)$ from the leading term,
  depend only on $\dosc$, so
  \[
    \Pr\bigl(\angle(z_i(0),-z_i^*) < \alpha\bigr)
    \sim C(\dosc)\,\alpha^{\dosc-1}
    \qquad \text{as } \alpha \to 0,
  \]
  where $C(\dosc) > 0$.
\end{proof}

The proposition makes two distinct scaling claims that should not be conflated.
The exponential-in-$\dosc$ decay holds for every fixed $\alpha \in (0, \pi/2)$
and describes how the unstable hemisphere shrinks as the sphere dimension
grows. The small-$\alpha$ asymptotic $\sim C(\dosc)\,\alpha^{\dosc-1}$ is a
separate statement, holding at each fixed $\dosc \geq 2$, describing the
polynomial rate at which the cap probability vanishes as the threshold
$\alpha$ tightens.

\begin{figure}[t]
\centering
\includegraphics[width=0.65\linewidth]{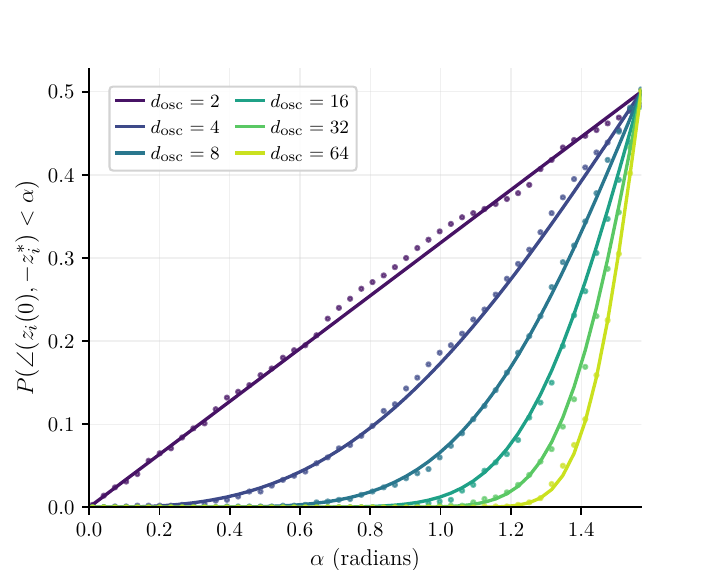}
\caption{\textbf{Validation of Proposition~\ref{prop:antipodal}.}  Markers show
  the empirical fraction of $N{=}1000$ uniform samples on $\mathbb{S}^{\dosc-1}$
  falling within angular distance $\alpha$ of a fixed pole, and lines show the
  closed-form prediction~\eqref{eq:capprob}. The empirical points fall on
    the predicted curve at this sample size across all $\dosc$ and $\alpha$, confirming that the probability of initializing
  $z_i(0)$ near the unstable equilibrium $-z_i^*$ decays sharply with $\dosc$.}
\label{fig:antipodal}
\end{figure}

Proposition~\ref{prop:degenerate} and Proposition~\ref{prop:antipodal} identify
two failure modes for finite-time ODE convergence at low $\dosc$. The
implementation reported here trains and infers using the analytic fixed point
$z_i^* = h_i/\|h_i\|$, so these failure modes affect only physical hardware
deployment, where they can be reduced by three levers. First, $\dosc$ is the
principal architectural lever (Table~\ref{tab:convergence}): the
  near-antipodal failure rate drops from $11.7\%$ at $\dosc=2$ to $0.81\%$ at
  $\dosc=32$, and overall convergence rises from $84.3\%$ to $98.3\%$. Second,
extending the integration window helps: at $\dosc=2$, growing $T_{\max}$ from
$30$ to $5000$ recovers convergence from $84.3\%$ to $98.8\%$ of trajectories. Third,
stronger coupling accelerates the dynamics: the gradient flow scales with
$\|h_i\|$, so increasing coupling magnitudes is equivalent to extending the
integration window. Slow-settling tokens are
  those with small driving norm $\norm{h_i}$, whose landscape is shallowest;
  they account for most of the residual non-convergence at finite $T_{\max}$.
  Coupling strength is a substrate design choice for physical oscillator
  arrays.

\section{Experiments}\label{sec:experiments}
All experiments use the analytic fixed point~\eqref{eq:fixedpoint} during
training, with gradients through the normalization handled automatically by
PyTorch autograd. The softmax baseline uses standard scaled dot-product
attention with identical architecture and hyperparameters. Query anchors are
initialized uniformly on $\mathbb{S}^{\dosc-1}$ and learned jointly with all
other parameters. We use $p{=}1$ throughout; ablation results appear in
Section~\ref{sec:readout}.

\subsection{Bidirectional tasks: keyword spotting and subject-verb
  agreement}\label{sec:bidirectional}
On bidirectional tasks, oscillator attention at $\dosc{=}2$ matches or exceeds
softmax. We verify this on two tasks: acoustic classification (keyword spotting)
and syntactic agreement (subject-verb agreement). Figure~\ref{fig:kws}
summarizes accuracy on both bidirectional tasks side by side; a detailed
analysis and comparison follow.

\begin{figure}[t]
\centering
\includegraphics[width=\linewidth]{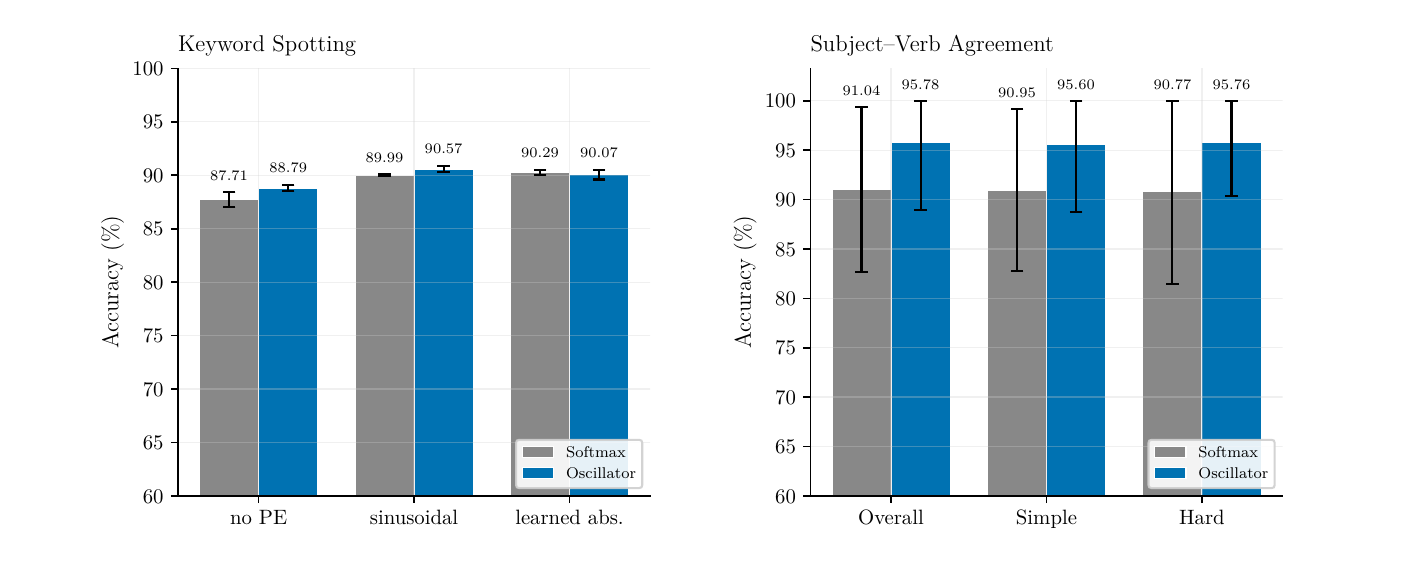}
\caption{\textbf{Bidirectional task accuracy.} \emph{Left:} KWS accuracy for
  softmax and oscillator attention under three positional-encoding conditions
  (none, sinusoidal, learned absolute); means over 5 seeds with $\pm$1 std
  whiskers. The learned anchors carry positional information, so the fair
  comparison is under matched encoding: with sinusoidal or learned absolute
  encoding the two mechanisms perform at parity. Without positional encoding,
  the oscillator's anchors give it a positional signal that softmax
  lacks. \emph{Right:} SVA at the minimum-hardware configuration
  ($d_{\rm model}{=}32$, 1 head, 1 layer): overall, simple, and hard accuracy,
  means over 50 seeds per mechanism with $\pm$1 std whiskers (truncated at
  $100\%$); both models use sinusoidal positional encoding. Training failures
  (hard accuracy below $85\%$): $3/50$ for the oscillator versus $12/50$ for
  softmax.}
\label{fig:kws}
\end{figure}

\subsubsection{Keyword spotting}

For keyword spotting (KWS), we train transformers with $d_{\rm model}{=}32$, 2
heads, 1 layer on 10-class Google Speech Commands~\citep{PW:18} (log-mel
spectrograms, 40 bins, $T{=}49$ frames, 5 seeds per condition).  Without
  positional encoding, the oscillator reaches $88.79 \pm 0.29\%$ versus
  $87.71 \pm 0.70\%$ for softmax, but this comparison is confounded: the anchors
  are learned per absolute position, so the oscillator model receives positional
  information that the PE-free softmax baseline lacks. A frame-permutation probe
  quantifies the effect: under random permutation of the $T{=}49$ input frames,
  the PE-free softmax model is unchanged ($\Delta = 0.00$~pp, as permutation
  invariance requires), while the PE-free oscillator loses $17.06 \pm
  3.64$~pp. This is comparable to the loss suffered by softmax with sinusoidal
  ($21.97$~pp) or learned absolute ($19.38$~pp) positional encoding.  The
  anchors act as a built-in positional encoding of standard strength. We
  therefore compare the mechanisms under matched positional information
  (Figure~\ref{fig:kws}, left): softmax reaches $89.99 \pm 0.09\%$ with
  sinusoidal and $90.29 \pm 0.24\%$ with learned absolute encoding; the
  oscillator reaches $90.57 \pm 0.29\%$ and $90.07 \pm 0.46\%$. Under matched
  conditions the two mechanisms perform at parity, which is the desired
  property: competitive accuracy at the minimal hardware configuration, not an
  advantage over~softmax.

A natural question is whether the oscillator dynamics actually determine the
classification, or whether the answer is already encoded in the value vectors
regardless of the attention pattern. To test this, we freeze $W_V$ at random
initialization and train only the oscillator parameters (anchor vectors and
coupling projections), removing any possibility that the value projection learns
to compensate for weak attention structure. The frozen-$W_V$ model reaches
$88.32 \pm 0.34\%$, only $0.47$ pp below the full model: a learned
  value projection is unnecessary at this configuration. The remaining
    components stay trainable, so no single ablation isolates the dynamics as
    the sole source of accuracy. The case that the dynamics are the operative
    component rests on the design-order and error-correction arguments of
    Remark~\ref{rem:whydynamics} together with the position-matched and
    permutation evidence above.

\begin{table}[t]
\centering
\begin{tabular}{lcc}
\toprule
Condition & Val acc (\%) & $\Delta$ vs softmax \\
\midrule
Softmax                                  & $87.71 \pm 0.70$ & --- \\
Oscillator $\dosc{=}2$ (full model)       & $88.79 \pm 0.29$ & $+1.08$ pp \\
Oscillator, frozen $W_V$                  & $88.32 \pm 0.34$ & $+0.61$ pp \\
\bottomrule
\end{tabular}
\caption{\textbf{KWS results, ablation set (no positional encoding).} 10-class Google Speech Commands,
  $d_{\rm model}{=}32$, 2 heads, 1 layer. Full model, frozen-$W_V$
  ablation, and softmax baseline, all under matched no-positional-encoding
  conditions over 5 seeds.  The
    position-matched comparison across positional-encoding conditions is
    reported in Figure~\ref{fig:kws} (left); without positional encoding, the
    anchors carry positional information that softmax lacks, so the $\Delta$
    column of this table overstates the mechanism difference.}
\label{tab:kws}
\end{table}

\begin{table}[t]
\centering
\begin{tabular}{lcc}
\toprule
Condition & Overall (\%) & Hard (\%) \\
\midrule
Softmax                                  & $91.04 \pm 8.32$ & $90.77 \pm 9.28$ \\
Oscillator $\dosc{=}2$ (full)            & $95.78 \pm 6.82$ & $95.76 \pm 5.37$ \\
Oscillator, frozen $W_V$                 & $95.11 \pm 7.14$ & $94.96 \pm 6.43$ \\
\bottomrule
\end{tabular}
\caption{\textbf{SVA results at minimum-hardware configuration}
  ($d_{\rm model}{=}32$, 1 head, 1 layer, $d_{\rm ff}{=}64$, 9 free
  oscillators per attention layer at $\dosc{=}2$). The ``hard'' split contains
  sentences whose distractor disagrees with the subject. All rows report
    means over 50 seeds per mechanism. Because failed runs skew the mean,
      we also give medians on the hard split: $95.49\%$ (softmax), $97.14\%$
      (full oscillator), and $96.96\%$ (frozen $W_V$). The large standard
      deviations reflect a bimodal seed distribution
    (Appendix~\ref{app:robustness}).  The frozen-$W_V$ ablation holds $W_V$ at
  random initialization, training only the oscillator parameters. At a larger
  standard architecture ($d_{\rm model}{=}64$, 2 heads, 2 layers), both methods
  reach $98\%$.}
\label{tab:sva}
\end{table}

\subsubsection{Subject-verb agreement}
We turn next to a task of very different structure: syntactic agreement in
natural-language sentences. For subject-verb agreement (SVA), we train on a
synthetic Linzen-style dataset~\citep{TL-ED-YG:16} of 40K/4K/4K sentences (e.g.,
``The keys on the table are/is''), with subject, distractor, and verb indices
recorded for attention analysis.  We use synthetically generated sentences to
ensure precise control over subject, distractor, and verb positions; the goal is
attention analysis rather than linguistic benchmark performance.  At the
standard transformer architecture ($d_{\rm model}{=}64$, 2 heads, 2 layers),
both softmax and oscillator attention reach $98\%$ accuracy. The task is fully
within the capacity of standard transformers at this scale.

To examine oscillator attention at minimum hardware scale, we evaluate a
constrained configuration: $d_{\rm model}{=}32$, single attention head, single
transformer layer, feed-forward dimension $d_{\rm ff} = 64$. This setup uses
only $9$ free oscillators per attention layer at $\dosc{=}2$
  ($n_h \times T \times (\dosc-1) = 1 \times 9 \times 1$, with all hard
  sentences $T{=}9$ tokens). At this configuration we compare the
  mechanisms over 50 seeds each (Table~\ref{tab:sva}; per-seed distribution in
  Appendix~\ref{app:robustness}). Training failures, defined as hard accuracy
  below $85\%$, occur in $3/50$ oscillator runs versus $12/50$ softmax runs,
  a difference unlikely to arise by chance (two-sided $p = 0.023$). Hard accuracy is $95.76 \pm 5.37\%$ for
  the oscillator (median $97.14$) versus $90.77 \pm 9.28\%$ for softmax (median
  $95.49$). All fifteen failures, in both mechanisms, are partial collapses in
  the $69$--$85\%$ band: the mechanisms share a failure mode and differ in its
  rate, by a factor of four. Successful runs of the two mechanisms remain within
  $1$--$2$~pp of each other. A sensitivity sweep over learning rate and number
  of training epochs distinguishes the two failure behaviors: the oscillator's
  only weak cell (lr $2.5\!\times\!10^{-4}$, 20 epochs) is eliminated simply by
  training longer ($0/5$ failures at 40 epochs), while softmax's failures at lr
  $10^{-3}$ persist at both 20 and 40 epochs. To verify that the oscillator
dynamics drive this configuration's accuracy, we train the oscillator with $W_V$
frozen at random initialization (third row of Table~\ref{tab:sva}).  The
frozen-$W_V$ model reaches $94.96 \pm 6.43\%$ hard accuracy,
  indistinguishable from the full model ($-0.80$~pp): neither the failure
    counts nor the accuracy distributions differ significantly
    (Appendix~\ref{app:robustness}). A learned value transformation is unnecessary at this
  configuration, paralleling the KWS finding. As there, this ablation scopes the
  role of $W_V$; it does not by itself attribute the accuracy to the dynamics.

The SVA task also reveals a structural difference in attention
geometry. Table~\ref{tab:attention} shows verb-to-subject versus
verb-to-distractor attention weights aggregated across all hard validation sentences
($n{=}1153$, aggregated over 5 seeds). The subject-distractor gap is consistently positive for
oscillator across seeds and across correct/wrong predictions, showing that the
oscillator reliably prioritizes the subject position; the gap is in fact
slightly larger on wrong predictions ($+0.19$) than on correct ones
($+0.12$). These statistics are computed over the five training runs used in this
analysis, trained at the minimum-hardware configuration, which reach
approximately $97\%$ accuracy on hard sentences. The errors are structural: essentially all wrong predictions
involve the words \emph{fish} or \emph{deer} as the subject, number-invariant
nouns whose singular and plural forms are identical in English (\emph{one fish,
  two fish}; \emph{one deer, two deer}). The oscillator correctly attends to the
subject, but the subject's embedding carries no morphological information
distinguishing singular from plural. The resulting classification is driven to
near-zero logit magnitude ($\overline{|\ell|} = 0.92$ on wrong vs.\ $6.45$ on
correct), landing on the wrong side of the decision boundary on close
calls. Softmax's error structure differs by seed. The four softmax seeds that
train successfully ($\geq 94\%$ hard accuracy) show error patterns similar to
oscillator: gap $+0.23$ on wrong predictions, with $77\%$ of the
attention-maintained errors (verb-to-subject minus verb-to-distractor attention
$\geq 0.1$) having close-call logits ($|\ell| < 0.5$).  The catastrophic-failure
seed (one of the five training runs analyzed here,
  $78.14\%$ hard accuracy) shows true attention collapse: gap $-0.08$ on wrong
predictions, attention favoring the distractor over the subject. The five-seed
mean reported in Table~\ref{tab:attention} (subject focus $+0.35 \to +0.16$,
correct $\to$ wrong) reflects the catastrophic seed's $-0.08$ averaged against
the $+0.23$ of the four successful seeds.  When training succeeds, both mechanisms
fail on similar intrinsically ambiguous cases; the accuracy difference
at this configuration is driven by the difference in failure rates rather
than by a difference in error mechanism.

\begin{table}[t]
\centering
\begin{tabular}{llccc}
\toprule
Attention & Split & $a_{\rm vs}$ & $a_{\rm vd}$ & Gap \\
\midrule
Softmax    & correct & $0.59\pm0.16$ & $0.24\pm0.15$ & $+0.35$ \\
Softmax    & wrong   & $0.23\pm0.14$ & $0.06\pm0.06$ & $+0.16$ \\
Oscillator & correct & $0.20\pm0.11$ & $0.08\pm0.04$ & $+0.12$ \\
Oscillator & wrong   & $0.26\pm0.09$ & $0.07\pm0.04$ & $+0.19$ \\
\bottomrule
\end{tabular}
\caption{\textbf{Where the verb's attention goes, on correct and on
      wrong predictions.} At the minimum-hardware configuration we train five
    runs of each mechanism, differing only in the random
    seed. Each run contributes the mean over its own $1153$ hard sentences, and
    we average those five means. $a_{\rm vs}$ and $a_{\rm vd}$ quantify the
    attention the verb pays to the subject and to the distractor; Gap is their
    difference, and a positive gap means attention goes to the right word. On
    correct predictions both mechanisms have a positive gap. The oscillator
    keeps one on wrong predictions too ($+0.19$), so its errors are not caused
    by attention going to the wrong word. For softmax, four of the five runs
    behave the same way ($+0.23$), while the fifth, which failed to train
    ($78.14\%$ accuracy), reverses to $-0.08$.}
\label{tab:attention}
\end{table}

\begin{figure}[t]
\centering
\includegraphics[width=\linewidth]{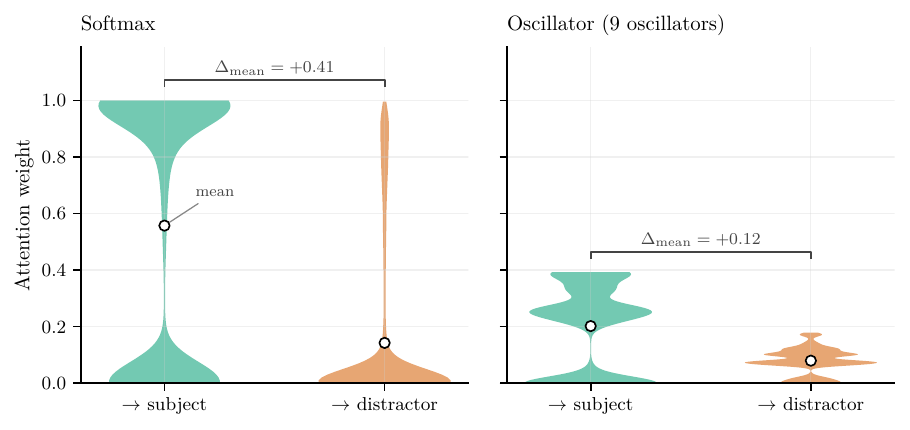}
\caption{\textbf{Verb attention distributions across all distractor test
    sentences ($n{=}2415$) at the minimum-hardware configuration, aggregated
    over 5 seeds.}  Softmax shows broad distributions reflecting high seed-level
  variability (one of the five training runs failed to converge above
    $80\%$; the displayed distribution combines stable and unstable runs). Oscillator
  shows tighter distributions with a consistent positive subject preference
  ($a_{\rm vs} > a_{\rm vd}$); over 50 seeds the oscillator's failure rate is
  lower ($3/50$ versus $12/50$).}
\label{fig:sva}
\end{figure}

Figure~\ref{fig:sva} shows verb attention distributions across all distractor
test sentences ($n{=}2415$).  Softmax distributions are broad due to high
between-seed variability, with the one training run of the five that
failed to converge pulling the
subject-attention distribution toward lower values. Oscillator distributions are
tighter and show a consistent subject preference, with anchor positions spread
broadly across the unit circle, indicating the mechanism uses the available
representational geometry rather than collapsing to a degenerate fixed point.

\subsection{Causal language modeling}\label{sec:lm}
We evaluate causal language modeling on two datasets: WikiText-2 (word-level,
vocabulary 10K, $T{=}50$, a standard LM benchmark) and
TinyStories~\citep{RE-YL:23} ($\sim$2M short stories, vocabulary 8K, simpler
narrative structure). The same transformer backbone is used throughout
($d_{\rm model}{=}128$, 4 heads, 2 layers), with context length $T{=}50$
for WikiText-2 and $T{=}128$ for TinyStories. Table~\ref{tab:lm} reports
validation perplexity at $\dosc \in \{2, 4, 8, 16, 32\}$; both datasets show the
oscillator--softmax gap shrinking as $\dosc$ grows. Every configuration
uses five seeds, including the softmax baselines.
On WikiText-2, the gap decreases from $+11.7$ PPL at $\dosc{=}2$ to $+3.6$
PPL at $\dosc{=}32$. On TinyStories, whose simpler narrative structure
demands less simultaneous disambiguation, the gap is smaller at every
$\dosc$ and closes faster, from $+2.24$ PPL at $\dosc{=}2$ to $+0.4$ PPL at
$\dosc{=}32$, approaching practical parity. Replacing the
analytic fixed point with the exact finite-horizon solution of the Lohe ODE at
$T_{\max}=30$ at inference recovers the same performance within $0.13$ PPL
(Section~\ref{sec:convergence}). Figure~\ref{fig:demo} illustrates the
oscillator dynamics on a TinyStories sentence: the active token's trajectory
converges toward its fixed point under learned coupling springs anchored to
context tokens.

\begin{table}[t]
\centering
\begin{tabular}{llccc}
\toprule
Dataset      & $\dosc$ & Oscillator PPL    & Softmax PPL & $\Delta$  \\
\midrule
WikiText-2   & $2$     & $110.67 \pm 0.51$ & $99.00 \pm 0.48$ & $+11.67$  \\
WikiText-2   & $4$     & $107.20 \pm 0.55$ & $99.00 \pm 0.48$ & $+8.20$   \\
WikiText-2   & $8$     & $105.23 \pm 0.32$ & $99.00 \pm 0.48$ & $+6.23$   \\
WikiText-2   & $16$    & $103.78 \pm 0.49$ & $99.00 \pm 0.48$ & $+4.78$   \\
WikiText-2   & $32$    & $102.56 \pm 0.37$ & $99.00 \pm 0.48$ & $+3.56$   \\
\midrule
TinyStories  & $2$     & $10.91 \pm 0.07$  & $8.67 \pm 0.04$  & $+2.24$   \\
TinyStories  & $4$     & $10.28 \pm 0.08$  & $8.67 \pm 0.04$  & $+1.61$   \\
TinyStories  & $8$     & $9.78 \pm 0.06$   & $8.67 \pm 0.04$  & $+1.11$   \\
TinyStories  & $16$    & $9.35 \pm 0.05$   & $8.67 \pm 0.04$  & $+0.68$   \\
TinyStories  & $32$    & $9.11 \pm 0.02$   & $8.67 \pm 0.04$  & $+0.44$   \\
\bottomrule
\end{tabular}
\caption{\textbf{Causal LM perplexity as a function of the oscillator
    dimension.} Validation perplexity on WikiText-2
  ($\dosc \in \{2, 4, 8, 16, 32\}$) and TinyStories
  ($\dosc \in \{2, 4, 8, 16, 32\}$), compared to a softmax baseline. All models use
  the same backbone ($d_{\rm model}{=}128$, 4 heads, 2 layers) and analytic
  fixed-point inference. The oscillator-versus-softmax gap closes
  monotonically as $\dosc$ grows. Mean and standard deviation across
  five seeds for every configuration, softmax baselines included; the
  held-out $\dosc{=}64$ configurations use ten seeds and are reported with
  the scaling analysis (Figure~\ref{fig:scaling}).}
\label{tab:lm}
\end{table}

\begin{figure}[t]
\centering
\includegraphics[width=\linewidth]{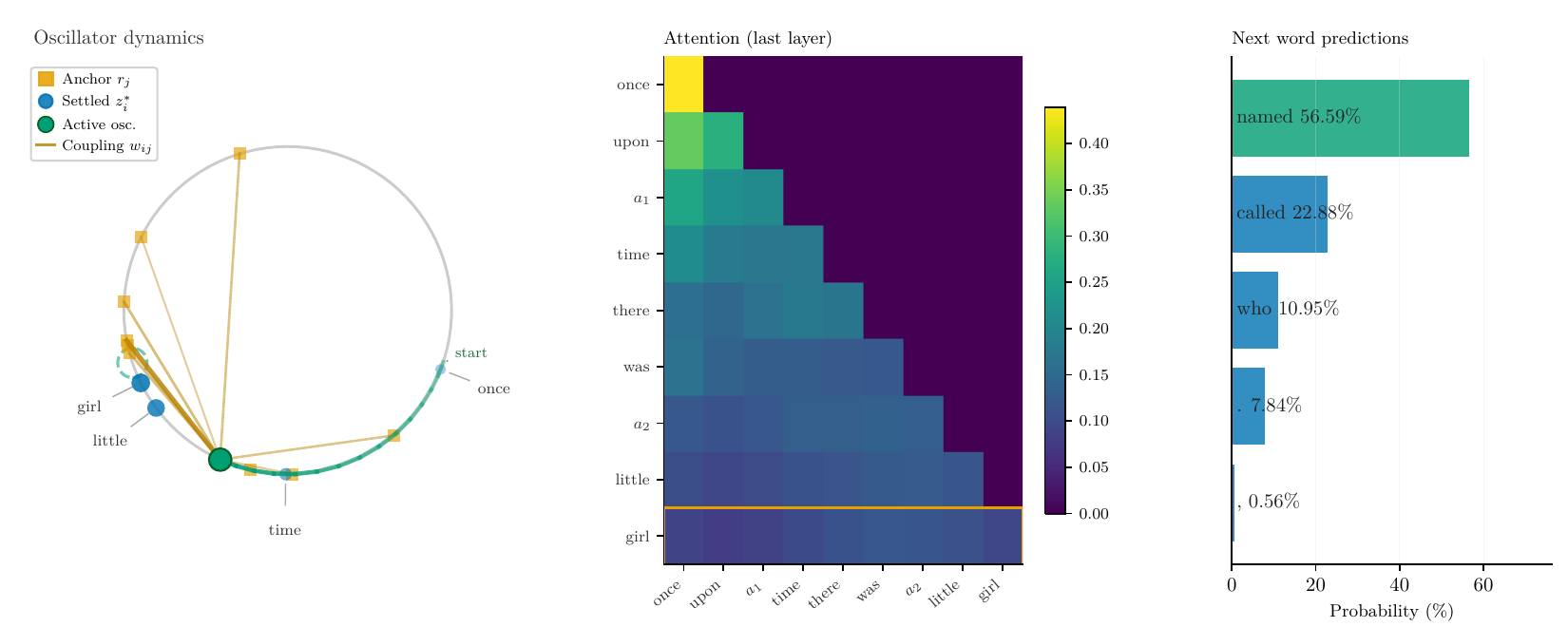}
\caption{\textbf{Oscillator dynamics on a TinyStories sentence.}  \textit{Left:}
  token oscillators on the unit circle; amber squares mark anchor positions
  $r_j$, dark markers show settled free-oscillator fixed points $z_i^*$, the
  active free oscillator $z_i$ traces a trajectory toward its fixed point
  $z_i^* = h_i/\|h_i\|$ along the trail, and lines indicate coupling weights
  $w_{ij}$.  \textit{Top right:} attention heatmap (last layer, head average).
  \textit{Bottom right:} next-word probability distribution; the prediction
  ``named'' receives $56.59\%$ probability.}
\label{fig:demo}
\end{figure}

The root cause of the oscillator-versus-softmax PPL gap is a dimensional
bottleneck. Both softmax and oscillator attention produce scalar attention
weights per token pair; the difference is the dimensionality available to
express the attention pattern across many pairs. Softmax computes
$\alpha_{ij}^{\text{sm}} \propto \exp\bigl((F e_i)^\top(G
e_j)/\sqrt{d_h}\bigr)$, so the attention pattern at position $i$ can vary along
all $d_h$ directions of the query-key product. Oscillator attention computes
$\alpha_{ij} \propto 1 + z_i^{*\top} r_j$ with the settled oscillator
constrained to the anchor span: $z_i^* = h_i / \norm{h_i}$ where
$h_i = \sum_l w_{il}\, r_l$. The attention pattern at position $i$ therefore
lives in the $\dosc$-dimensional manifold spanned by the anchors. When $\dosc$
is small, the attention matrix has bounded effective rank: rows
$\{\alpha_{ij}\}_j$ are tied to one another through the shared anchor geometry,
and rotations of $h_i$ that softmax could distinguish collapse to the same
oscillator state. As $\dosc$ grows, the anchor span enlarges and the attention
rows become more independent. On tasks where one dimension of similarity is
sufficient (KWS; SVA at the standard architecture), the dimensional constraint
is not binding and $\dosc{=}2$ achieves parity with softmax. At constrained capacity (SVA minimum-hardware configuration), the
sphere constraint is associated with a significantly lower training-failure
rate ($3/50$ versus $12/50$ seeds). On tasks requiring simultaneous
tracking of many contextual relationships (causal LM), higher $\dosc$ is
needed. These are different operating points: the low-dimensional regime
targets always-on classification, where $\dosc{=}2$ reaches parity, while
causal language modeling is a harder setting where the gap closes smoothly
with $\dosc$, at the rate the scaling relation quantifies. We verified that
the gap persists under three alternative interventions (learned
position-dependent phase offsets, per-head coupling amplification,
and independent scaling of attention heads and $d_{\rm model}$), none of which
close the gap independently of $\dosc$
(Appendix~\ref{app:alternative_ablations}). Appendix~\ref{app:alternative_ablations} supports this diagnosis
with measurements: the effective rank of the similarity operator
approaches but stays below its $\dosc{+}1$ ceiling on trained models,
truncating a trained softmax
model to comparable rank at inference is far more damaging than training
under the constraint, and a softmax variant trained with matched score rank
recovers part, but not all, of the difference. The rank ceiling constrains
small $\dosc$, and it is not the only constraint; the readout's selectivity
contributes as well (Section~\ref{sec:readout}). We conjecture that parity
is reached when $\dosc$ matches the effective rank of the softmax attention
matrix on the task, a question that we leave as the subject of future
investigation.

\begin{figure}[t]
\centering
\includegraphics[width=0.95\linewidth]{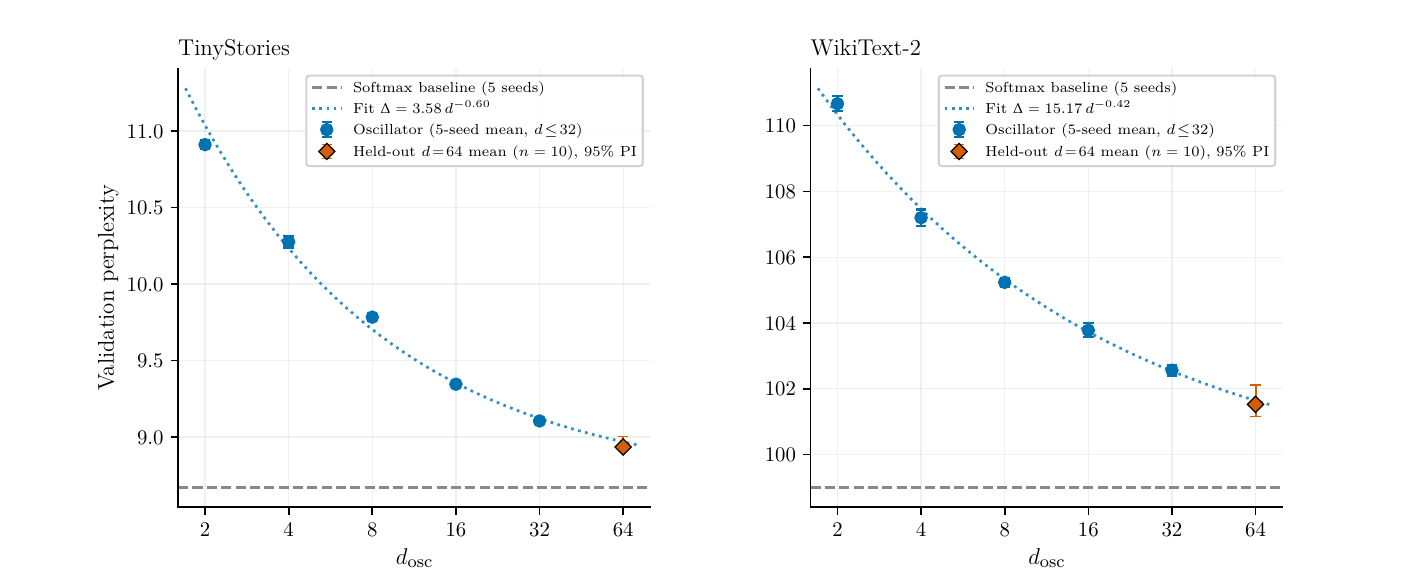}
\caption{\textbf{Perplexity gap vs.\ oscillator dimension, with a held-out
      prediction test.} Validation perplexity on TinyStories (left) and
    WikiText-2 (right). Circles: oscillator attention at
    $\dosc \in \{2, 4, 8, 16, 32\}$, five-seed means with whiskers giving the
      standard error across seeds; dashed horizontal line: softmax baseline (five
    seeds). Dotted curves: power-law fits $\Delta \approx C\cdot\dosc^{-\alpha}$
    to the $\dosc \le 32$ points, with $C \approx 3.58$, $\alpha \approx 0.60$
    on TinyStories and $C \approx 15.17$, $\alpha \approx 0.42$ on
    WikiText-2. Diamonds: held-out $\dosc{=}64$ configurations, excluded from
    the fits; each diamond is the mean of ten seeds, and its whisker spans the
    95\% prediction interval for that mean, combining fit uncertainty and
    observation noise.  Both held-out means fall inside their intervals.}
\label{fig:scaling}
\end{figure}

Figure~\ref{fig:scaling} shows that the perplexity gap on both datasets
  decays consistently with a power law. The fits over
  $\dosc \in \{2, 4, 8, 16, 32\}$ give $\Delta \approx 3.58 \cdot \dosc^{-0.60}$
  on TinyStories ($R^2 = 0.994$) and $\Delta \approx 15.17 \cdot \dosc^{-0.42}$
  on WikiText-2 ($R^2 = 0.997$); resampling the per-seed perplexities gives
    $95\%$ intervals for the exponents of $[0.572, 0.627]$ on TinyStories and
    $[0.386, 0.463]$ on WikiText-2. The study
  uses a uniform seed design: five seeds per configuration, baselines included,
  and ten seeds at each held-out test point. Models at $\dosc = 64$, a dimension
  excluded from the fits, are compared with each fit's extrapolation using a
  95\% prediction interval, which combines the uncertainty of the fit and the
  seed noise of the observation. On TinyStories the observed ten-seed mean gap
  is $0.264$ against a predicted $0.296$, inside the 95\% prediction interval
  $[0.260, 0.332]$. On WikiText-2 it is
  $2.523$ against a predicted $2.635$, inside $[2.159, 3.111]$. The fitted relation therefore predicts the held-out dimension on
  both datasets, within the stated intervals. We conjecture that the perplexity
gap obeys a scaling of the form $\Delta(\dosc) \;\approx\; C\,\dosc^{-\alpha}$
for task-dependent constants $C, \alpha > 0$.
A precise theoretical bridge between the two mechanisms is hindered by
the fact that softmax and oscillator attention compute attention weights through
structurally different operations (exponentiated similarity versus shifted
cosine on the sphere), so the apparent convergence in our data remains an empirical
observation.

\begin{table}[t]
\centering
\begin{tabular}{lccc}
\toprule
Task / $\dosc$                 & $p{=}1$          & $p{=}2$          & $p{=}4$          \\
\midrule
KWS ($\dosc{=}2$, acc \%)       & $88.79 \pm 0.29$ & $89.46 \pm 0.68$ & $90.02 \pm 0.73$ \\
TinyStories ($\dosc{=}2$, PPL)  & $10.91 \pm 0.07$ & $10.42 \pm 0.06$ & $10.17 \pm 0.11$ \\
TinyStories ($\dosc{=}8$, PPL)  & $9.78 \pm 0.06$  & $9.18 \pm 0.02$  & $8.98 \pm 0.07$  \\
\bottomrule
\end{tabular}
\caption{\textbf{Readout sharpening.}  First row is KWS accuracy (\%) at
  $\dosc{=}2$; remaining rows are TinyStories validation PPL at $\dosc{=}2$
  and $\dosc{=}8$. All entries use five seeds; the $p{=}1$ column reproduces the
  main-table values (Tables~\ref{tab:kws} and~\ref{tab:lm}). Sharpening improves
  performance over $p{=}1$, with the largest gains at $p{=}4$.}
\label{tab:readout}
\end{table}

\subsection{Readout sharpening}\label{sec:readout}
The readout exponent $p$ plays a role analogous to inverse temperature in
softmax: larger $p$ sharpens the attention distribution toward the strongest
cosine matches, while smaller $p$ flattens it toward a uniform distribution.  As
a software-side optimization, $p$ consistently improves performance across all
tasks tested (Table~\ref{tab:readout}). On KWS at $\dosc{=}2$, raising $p$ from
$1$ to $2$ gives $89.46 \pm 0.68\%$ versus $88.79 \pm 0.29\%$ at $p{=}1$
($+0.67$ pp), and $p{=}4$ reaches $90.02 \pm 0.73\%$ ($+1.23$
pp). These runs use the no-positional-encoding ablation, so they are
    compared within the oscillator mechanism and not against the softmax
    baseline.  On causal language modeling, raising $p$ from $1$ to $4$
improves perplexity by $0.74$ at $\dosc{=}2$ and $0.80$ at $\dosc{=}8$.
In both settings, a modestly sharpened readout improves task accuracy by
suppressing weak background alignments that would otherwise dilute the attention
signal. Sharpening is therefore a deployment-mode choice: $p{=}1$ is the
  readout the substrate performs directly, and $p{>}1$ is available when digital
  post-processing is acceptable.

Sharpening changes selectivity but not rank. On trained language models
  the normalized attention-row entropy is high for the oscillator (median above
  $0.90$, versus $0.60$ for softmax), and raising $p$ concentrates the attention
  mass. This is the low-entropy selectivity that \citet{hedgehog2024} identify
  as one of two properties softmax alternatives often lose. The shifted-cosine
  readout preserves monotonicity in the query-key score because the weight on
  token $j$ is monotone in $z_i^{*\top} r_j$. The rank ceiling is preserved: the
  $\dosc{+}1$ bound on the pre-mask similarity operator is fixed by the
  oscillator dimension (Appendix~\ref{app:alternative_ablations}) and does not
  depend on $p$.

\begin{table}[t]
\centering
\begin{tabular}{lcc}
\toprule
$\dosc$ & Converged (err $<0.01$) & App.\ antipodal \\
\midrule
2  & $84.3\%$ & $11.7\%$ \\
8  & $94.9\%$ & $3.4\%$  \\
32 & $98.3\%$ & $0.81\%$ \\
\bottomrule
\end{tabular}
\caption{\textbf{Convergence of the oscillator dynamics on TinyStories.}
    Trajectories integrated with RK45 to $T_{\max}=30$. Convergence improves
    with $\dosc$: trajectories that begin near the unstable equilibrium become
    rare, as predicted by Proposition~\ref{prop:antipodal}. The sample is 32
    validation sequences of 128 tokens, giving $2{,}895$
    token positions with one oscillator per position per head and five
    initializations each: $57{,}900$ trajectories per model.}
\label{tab:convergence}
\end{table}

\begin{figure}[t]
\centering
\includegraphics[width=\linewidth]{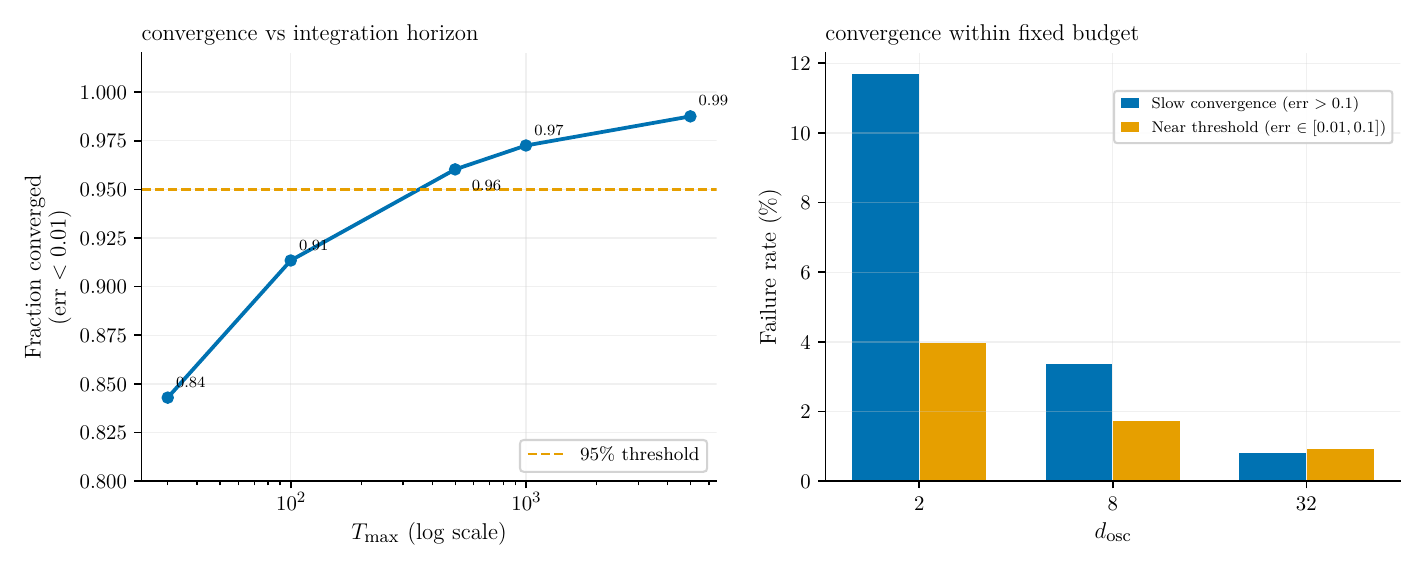}
\caption{\textbf{ODE convergence verification.}  \emph{Left:} At $\dosc{=}2$,
  the fraction of trajectories converging (err\,$<$\,0.01) grows with
  integration horizon $T_{\max}$, reaching $98.8\%$ at $T_{\max}{=}5000$; all
  trajectories eventually converge given sufficient time, consistent
  with Theorem~\ref{thm:uniqueness}.  \emph{Right:} At fixed budget
  $T_{\max}{=}30$, and over the same $57{,}900$ trajectories,
  convergence failure rates decrease strongly with $\dosc$. The two
    failure modes are distinct: an apparent-antipodal trajectory starts close to
    $-z_i^*$ and escapes slowly however strong the driving, whereas a
    small-$\|h_i\|$ trajectory has weak driving and so a shallow landscape
    everywhere. Both become exponentially rare with $\dosc$, consistent with
    Propositions~\ref{prop:degenerate} and~\ref{prop:antipodal}.}
\label{fig:convergence}
\end{figure}

\subsection{ODE convergence verification}\label{sec:convergence}
Theorem~\ref{thm:uniqueness} guarantees convergence asymptotically; here
we verify it numerically in finite time and characterize its dependence on
$\dosc$. We run the Lohe ODE using an adaptive RK45 integrator
(\texttt{scipy.solve\_ivp}, $rtol=10^{-4}$, $T_{\max}=30$) on the trained
TinyStories models and measure convergence error $\norm{z_{\rm ODE}^*-z_i^*}_2$
against the analytic fixed point. Convergence improves strongly with $\dosc$
(Table~\ref{tab:convergence}), consistent with
Proposition~\ref{prop:antipodal}. The ``apparent antipodal'' classification in
Table~\ref{tab:convergence} is a finite-time threshold determined by
$T_{\max}{=}30$ (trajectories that have not converged within the
integration window), not the measure-zero set of initial conditions lying
  exactly at the unstable equilibrium of Proposition~\ref{prop:antipodal}; the
empirical count is the more conservative quantity. Residual failures at each
$\dosc$ reflect initializations near the unstable equilibrium where convergence
within the fixed budget $T_{\max}{=}30$ is insufficient, not convergence to a
wrong attractor.  The left panel of Figure~\ref{fig:convergence} confirms that
all cases eventually converge given sufficient integration time: extending
$T_{\max}$ from $30$ to $5000$ for $\dosc{=}2$ recovers convergence
monotonically ($84.3\%\to91.3\%\to96.0\%\to97.3\%\to98.8\%$).  Sequential
initialization (starting each oscillator from the previous token's analytic
fixed point perturbed by Gaussian noise of scale $0.05$) reduces mean error
$4.6\times$ at $\dosc{=}2$ and costs nothing extra. At $\dosc{=}32$ the advantage
is negligible because random initialization already avoids slow-convergence with
high probability.

Beyond pointwise convergence, on the trained WikiText-2 $\dosc{=}2$ model, the
exact finite-horizon ODE solution recovers the analytic fixed-point perplexity
within $0.13$ PPL. Sequential initialization (starting each token's free oscillator
from the previous token's integrated state at $T_{\max}$) closes the residual
gap to $0.04$ PPL. The strong driving forces produced by the trained coupling
weights ensure near-perfect convergence within $T_{\max}{=}30$. The convergence
results are also robust to the integration scheme: RK45 agrees with the analytic
fixed point at the state and task-metric level, and forward Euler agrees except
at rare large-$\norm{h_i}$ tokens, where its fixed step becomes unstable
(Appendix~\ref{app:robustness}). Physical hardware running continuous-time
dynamics could close the gap further.

\section{Discussion}\label{sec:discussion}
On bidirectional tasks where the attention structure is simple (grouping voiced
frames in KWS, identifying a single grammatical dependency in SVA), scalar
Kuramoto dynamics are sufficient and achieve parity with softmax. The
  sphere constraint and bounded cosine readout are associated with a
  significantly lower training-failure rate at constrained capacity and with
  geometrically structured attention maps. On causal language modeling, where
each token must simultaneously track many independent contextual relationships,
the $\dosc$-dimensional anchor subspace constrains the rank of the attention
pattern; higher $\dosc$ relaxes this constraint, and the gap closes
consistently with a power law. Existing efforts to reduce the cost of
attention on digital hardware (linear attention~\citep{AK-AV-NP-FF:20},
Performer~\citep{KC-VL-DD-XS-AG-TS-PH-AM-LK-DB-LK-AW:21},
Linformer~\citep{SW-BLI-MK-HF-HM:20}, sparse and local attention
variants~\citep{RC-SG-AR-IS:19,IB-MEP-AC:20}, and relative position encodings
such as RoPE~\citep{JS-MA-YL-SP-WB-YL:24} and ALiBi~\citep{OP-NAS-ML:22})
reduce asymptotic cost or impose inductive biases but remain within the
softmax framework, addressing sequence length scaling rather than the energy
cost of the attention operation itself. The von Neumann memory hierarchy
remains the dominant energy sink at the small sequence lengths characteristic
of edge devices, where their asymptotic advantages have not yet
materialized. Our mechanism addresses a different axis: the goal is to
establish that a physically realizable attention mechanism can achieve
accuracy within a predictable and controllable range of softmax. The $\dosc$
scaling law provides the control: a system designer can choose $\dosc$ to
match the accuracy requirements of the task. The mechanism is also orthogonal
to algorithmic choices: sparsity patterns and positional encodings can be
combined with oscillator attention, since they modify which weights $\Wij$ are
nonzero rather than how attention weights are computed from token embeddings.

\paragraph{Computational cost.} Softmax attention and oscillator
  attention differ only in how the attention weights are formed, and the
  surrounding computations are the same for both mechanisms. Both end with a sum
  over all tokens, but softmax exponentiates every score before summing, and an
  exponentiated sum has no efficient physical analog, while the oscillator sums
  linear terms, which an array can do by summing currents or
  charge. Table~\ref{tab:cost} gives the reduction in operations for that step,
  over four implementations that trace how much of it is physical, from none to
  all.  On the two edge tasks the reduction grows at every step, from below a
  factor of two in software to about twelve when only the coupling nonlinearity
  is left digital. TinyStories is the exception at the top of the table: its
  readout dominates at $\dosc{=}8$, so the oscillator costs more than softmax in
  software and about the same once the equilibration is physical, and it joins
  the other two only once the readout moves into the array. Which of the four
  implementations a device permits is a property of that device.
  The counts assume ReLU coupling. The trained models use softplus, and
  Appendix~\ref{app:alternative_ablations} shows that the two perform the same;
  with softplus the reduction reaches only $1.2\times$
  (Table~\ref{tab:cost_softplus}). These are first-order counts; fused kernels,
  quantization and streaming formulations all change the arithmetic, and the
  comparison rests on the structural difference rather than the exact figures.
Our training and task evaluations compute the equilibrium from its closed
form~\eqref{eq:fixedpoint}. In an analog implementation the equilibration has no
operation count and its cost is time. The dynamics settle within a horizon of
about $5$ on the bidirectional tasks and $30$ on TinyStories
(Appendix~\ref{app:robustness}). That horizon is dimensionless, a property of
the dynamics rather than of any device. The equilibrium is reached by a physical
process, so its convergence time and its energy consumption depend on the
specific implementation and its physical constraints; for that reason we report
no absolute energy figures and compare only the arithmetic a digital
implementation must perform, and the tolerances a physical realization would
need are projected in Appendix~\ref{app:robustness}.

\begin{table}[t]
\centering
\begin{tabular}{lccc}
\toprule
Reduction in attention-layer operations relative to softmax & KWS & SVA & TinyStories \\
\midrule
Softmax (reference)                            & $1.0$  & $1.0$  & $1.0$  \\
\midrule
Oscillator, all digital                        & $1.7$  & $1.6$  & $0.6$  \\
Oscillator, equilibration physical             & $2.4$  & $2.4$  & $1.1$  \\
Oscillator, equilibration and readout physical & $4.0$  & $4.1$  & $4.0$  \\
Oscillator, affine sum and division physical   & $12.0$ & $11.9$ & $12.0$ \\
\bottomrule
\end{tabular}
\caption{\textbf{Reduction in attention-layer operations relative to
      softmax.} Each row is an implementation of the step where the two
    mechanisms differ, and the figure is the factor by which it reduces the
    operation count. The different implementations move successively more of the
    step into the array: none, then the equilibration, then the inner products
    of the readout, then the affine sum and division. The last leaves only the
    coupling nonlinearity digital, so its factor approaches the weighting given
    to an exponential plus the two normalization terms the oscillator no longer
    pays, which is why all three tasks land near the same figure. Two
    assumptions enter: the counts assume ReLU coupling, and an exponential is
    counted as equivalent to ten multiply-adds with every other operation as
    one, following \citet{MH:14} and \citet{HP-YP:26}. Everything around the
    step is identical for both mechanisms and is not counted, nor is the cost of
    writing the couplings into the array, which depends on the
    device. Configurations: KWS ($T{=}49$, $n_h{=}2$, $\dosc{=}2$), SVA
    ($T{=}9$, $n_h{=}1$, $\dosc{=}2$), TinyStories ($T{=}128$, $n_h{=}4$,
    $\dosc{=}8$, causal). The same table with softplus coupling is
    Table~\ref{tab:cost_softplus}; absolute counts and the counting convention
    are in Appendix~\ref{app:details}.}
\label{tab:cost}
\end{table}

At $\dosc{=}2$, oscillator attention requires $T$ scalar oscillators per head,
  i.e.\ $n_h \times T \times (\dosc-1)$; this counts the active free
  oscillators, since the anchors are fixed forcing points set by learned
  parameters rather than evolving elements. This is $98$ total for the KWS
setting (2 heads, $T{=}49$). At $\dosc{=}32$, approximately $6{,}200$
  oscillators are needed for the WikiText-2 setting ($n_h{=}4$,
  $T{=}50$). These phase-oscillator counts assume substrates with native
$\dosc$-dimensional parameters. At $\dosc{=}2$, hardware realizations are
well-developed across multiple substrates: mechanical oscillators, Josephson
junctions, MEMS resonators, and electrical LC tanks. For $\dosc > 2$, oscillator
hardware is less mature: some substrates with intrinsic multi-dimensional order
(e.g., vector laser modes, three-dimensional spin systems) appear in principle
compatible with Lohe dynamics on $\mathbb{S}^{\dosc-1}$, but established
hardware demonstrations at higher $\dosc$ remain open. The $\dosc{=}2$ case is
therefore the immediate hardware-deployment target; higher $\dosc$ is realized
in software in this work and is an open hardware question. Across these
substrates, energy per oscillator is bounded by device physics, so the energy
advantage scales with the oscillator count. The fitted scaling
  $\Delta(\dosc) \approx C \cdot \dosc^{-\alpha}$ provides a
  task-complexity-to-oscillator-budget design rule for whichever substrate is
  used. Other physical-computing approaches for attention exist but are tied to
specific substrates: spiking
transformers~\citep{RJZ-QZ-JKE:23,CL-TL-JX-CG-ZL-CZ-XZ-XH:23,MY-JH-ZZ-LY-YT-BX-GL:23}
achieve 2--8$\times$ energy reduction but retain the switching costs of
transistor circuits; photonic transformer accelerators~\citep{zhu2024lightening}
use optical interference for dynamic matrix products in attention but still
require digital softmax and optical-to-electrical conversion between layers;
memristive accelerators implement in-memory matrix multiplications but incur ADC
overhead at every layer. To our knowledge, no existing physical-computing
approach for attention is substrate-independent.

The mechanism has a dynamical-systems lineage. The Lohe model~\citep{MAL:09}
generalizes Kuramoto to $\mathbb{S}^{\dosc-1}$; its synchronization has been
mostly studied for symmetric all-to-all coupling (see
e.g.~\citep{SC-MG-EO:19}). When the dynamics admit a gradient-flow structure
with a Lyapunov function, multi-agent consensus on $\mathbb{S}^n$ ($n \geq 2$)
achieves almost global convergence on any connected
graph~\citep{ROS:06-swarms,markdahl2018almost,lipton2021kuramoto}, a property
that fails on the circle ($n=1$) and is generally lost for non-gradient dynamics
(heterogeneous frequencies, asymmetric or matrix-weighted couplings), where
multistability is the norm. The fixed-query asymmetric setting (where anchors
are external forcing terms rather than free participants) inherits the
gradient-flow structure in its simplest form: a single oscillator under linear
forcing $h_i$, with Lyapunov function $V(z_i) = -z_i^\top h_i$ on
$\mathbb{S}^{\dosc-1}$. This admits a unique global minimum on the sphere,
yielding the uniqueness guarantee of Theorem~\ref{thm:uniqueness}. Earlier work
has used oscillator networks for pattern recognition via phase-locking in PLL
circuits~\citep{FCH-EMI:00b} and has analyzed the broader control-theoretic
structure of coupled-oscillator
networks~\citep{TM-GB-DSB-FP:22,YQ-AMN-DSB-FP:23}, but not as a transformer
attention primitive. Kuramoto oscillators have also been applied to associative
memory~\citep{TG-AO-ARV-MRS-FB-FP:26}, combinatorial optimization on CDW
arrays~\citep{JOB-TG-FP-AAB:25}, and, in AKOrN~\citep{TM-SL-AG-MW:25}, as
  neuron activations that bind features within the layers of a digitally
  executed network. AKOrN and our mechanism are complementary in what the
  oscillators compute, how they interact, and where the computation
  runs. AKOrN's oscillators interact through learned mutual coupling and
    run inside a digital forward pass. Our free oscillators do not couple to one
    another; they align independently to fixed anchor forcing, which is what
    yields the closed-form fixed point and the convergence guarantee. We
  replace the attention normalization itself with the equilibrium of this
  entrainment process and target execution on physical substrates. In
concurrent work, \citet{zhou2026emergence} also connect transformer attention
with Kuramoto dynamics, but in the opposite direction: they introduce a
dynamical temporal attention that modulates the emergent coherence of a
synchronizing oscillator network, whereas fixed-query oscillator attention uses
oscillator equilibration to implement the attention operation itself. Their
kernel attends over past states and retains softmax normalization, while ours
eliminates the global exponential normalization through the shifted-cosine
readout; the two works are best read as complementary explorations of the
attention--synchronization correspondence, for understanding synchronization
dynamics and for designing physically realizable attention mechanisms,
respectively.

A cluster of recent work couples attention to oscillator dynamics, along
  axes distinct from ours. Earlier, \citet{shende2026oscillators} replace the
  Neural-ODE solver of a continuous-time transformer with a linear
  damped-harmonic-oscillator analogy that admits a closed-form solution,
  modeling keys and values as damped driven oscillators and reading attention
  from a resonance integral. That line targets irregular-time-series modeling on
  conventional digital hardware and makes no substrate-independence or
  physical-realization claim. Concurrently with our submission,
  \citet{nunley2026kuramoto} places attention on the torus, updating each token
  by the tangent component of an attention-weighted circular mean, an adaptive
  Kuramoto coupling with rotary natural frequencies, while retaining a softmax
  score map. A follow-up~\citep{nunley2026frustrated} adds Kuramoto--Sakaguchi
  frustration to improve long-range copying. Both study synchronization
  attention as a software architecture evaluated on character-level language
  modeling. Separately, \citet{xiao2026kope} use Kuramoto synchronization as a
  neuro-inspired phase-encoding module to improve the learning efficiency of
  otherwise standard digital networks, enhancing attention rather than replacing
  its normalization and, like AKOrN, making no physical-substrate claim. Our aim
  is orthogonal: a substrate-independent mechanism whose only global operation
  is an affine readout, with the uniqueness and global-stability guarantee of
  Theorem~\ref{thm:uniqueness} transferring to every physical realization of the
  flow. That oscillator equilibration can implement attention is a finding
  shared across these efforts; the distinctions are the state space, whether the
  exponential normalization is removed, and whether the target is physical
  execution or a digital model.

Beyond engineered hardware, the framework connects to theories of
biological neural computation. The binding-by-synchrony
hypothesis~\citep{MB-SH-AD:10,PF:05} proposes that cortical
oscillations coordinate distributed representations through phase
alignment, with stable phase relationships encoding which features
belong together. The fixed-query mechanism is a formalization of
this idea: query anchors play the role of sustained reference
oscillations (analogous to gamma-band templates in cortical
attention), and free oscillators play the role of stimulus-driven
populations that phase-lock selectively to those references. The
positive-coupling constraint that guarantees convergence parallels
the predominantly excitatory long-range projections in cortex.
Whether this analogy can be made precise enough to inform
biologically plausible learning architectures (in which the
coupling weights $\Wij$ are updated by a Hebbian or
spike-timing-dependent plasticity rule rather than backpropagation)
is an open question we regard as promising.

Within the deep-learning literature, the most directly comparable framework is
the modern Hopfield
network~\citep{JJH:82a,HR-BS-JL-PS-MW-TA-LG-MH-DK-MK-GK-JB-SH:21}: both
mechanisms produce attention as fixed-point retrieval from a continuous
energy-based system over a finite set of reference patterns. The two differ in
both the structure of the energy landscape and the form of retrieval. Modern
Hopfield retrieval is an analytic one-step update (a softmax over similarities
to stored patterns), corresponding to a free-energy landscape that admits
multiple attractors (one per stored pattern), with retrieval converging to the
basin containing the query. Oscillator attention's energy landscape is sharper:
Theorem~\ref{thm:uniqueness} guarantees that the fixed-query gradient flow on
$\mathbb{S}^{\dosc-1}$ has a single stable attractor, so retrieval converges to
a unique fixed point regardless of initialization. The retrieval itself is
realized by physical continuous-time dynamics rather than a discrete software
step: trajectories trace an explicit gradient flow whose convergence rate is set
by the substrate's physical time constant, not by a floating-point
computation. This combination, an energy landscape collapsed to a single
attractor and retrieval realized as a physical gradient flow, makes oscillator
attention substrate-realizable while preserving the fixed-point-retrieval
semantics the Hopfield literature established as a useful frame for
attention. Fixed query anchors are also related to the prototype or codebook
attention of Set Transformers~\citep{JL-YL-JK-AK-SC-YWT:19}, where a fixed set
of learned seed vectors attends to the input. In Set Transformers the seeds are
free parameters updated by standard softmax attention, while our anchors are
fixed during inference and the free oscillators converge to them via physical
dynamics, with convergence guaranteed by Theorem~\ref{thm:uniqueness}. Taken
together, this positions oscillator attention as a candidate primitive for
physical computing: a mechanism with clean theoretical guarantees, predictable
scaling behavior, and a path to hardware realization that does not depend on
approximating digital arithmetic in analog devices.

Several directions remain open. Our language-model experiments use small
  transformers; scaling fixed-query oscillator attention to standard transformer
  sizes, and evaluating it on additional language-modeling datasets beyond the
  two studied here, would test whether the scaling behavior reported above
  persists.  On the hardware side, our robustness study covers the
  non-idealities we could simulate, but a full physical realization will face
  effects we do not model, e.g., signal-propagation delays across the coupling
  network and the cost of programming the anchor couplings, whose quantitative
  study we leave to a hardware-focused follow-up.


\section{Conclusion}
Fixed-query oscillator attention replaces the global exponential normalization and in-loop
reduction of softmax with the physical equilibration of a coupled oscillator
network on $\mathbb{S}^{\dosc-1}$. The fixed point is provably unique and
globally attractive from almost every initial condition whenever the weighted
anchor sum is nonzero, with degenerate cases vanishing exponentially in
$\dosc$. At $\dosc{=}2$, the mechanism matches softmax on KWS under matched
positional information, and on SVA at the minimum-hardware configuration it
fails to train significantly less often ($3/50$ versus $12/50$ seeds). On
causal language modeling, the gap closes consistently with a power law in
$\dosc$, and the fitted relation predicts a held-out $\dosc{=}64$
configuration on both datasets.

The broader claim is substrate independence. The Kuramoto synchronization
phenomenon arises naturally in electrical circuits, mechanical oscillators,
superconducting junctions, quantum materials, and neural tissue. Wherever it
arises, it can compute attention: equilibration and attention are described by
the same mathematics, so no engineering effort is spent approximating a digital
algorithm in analog hardware. This is what we mean by physical intelligence: a
computation that is a property of a
dynamical class rather than a specific device. Softmax remains the right choice
for digital hardware. For any physical substrate where Kuramoto-Lohe dynamics
can be realized, oscillator attention provides a principled alternative with
theoretical guarantees and characterized scaling behavior. The specific energy
and fabrication tradeoffs across different substrates are subjects for further
work.

\subsubsection*{Acknowledgments}
This material is based upon work supported in part by ARO award
W911NF-24-1-0228.

\bibliographystyle{unsrtnat}
\bibliography{./bib/paper-refs}

\appendix

  \section{Derivation of the Kuramoto equation from the Lohe model at
    $\dosc=2$}\label{app:kuramoto-derivation}
For completeness, we show how the scalar Kuramoto equation arises as the
$\dosc=2$ case of the Lohe model~\eqref{eq:lohe}. The derivation is short: a
single phase angle $\theta_i \in [0, 2\pi)$ fully captures the dynamics of a
unit vector $x_i \in \mathbb{S}^1$, because all the tangent motion projects
onto one direction.

Let $\dosc = 2$ and parametrize each oscillator by its phase angle:
\[
  x_i = (\cos\theta_i,\, \sin\theta_i) \in \mathbb{S}^1, \qquad
  \theta_i \in [0, 2\pi).
\]
The tangent space to $\mathbb{S}^1$ at $x_i$ is one-dimensional and spanned by
\[
  \tau_i := (-\sin\theta_i,\, \cos\theta_i) = \frac{\partial x_i}{\partial \theta_i}.
\]
By the chain rule, the time derivative of $x_i$ is
\[
  \dot x_i
  = \frac{\partial x_i}{\partial \theta_i}\,\frac{\mathrm{d}\theta_i}{\mathrm{d}t}
  = \tau_i\,\dot\theta_i,
\]
which lies along the tangent direction. The projector $(I - x_i x_i^\top)$
restricts any vector to the line spanned by $\tau_i$:
\[
  (I - x_i x_i^\top)\, v
  = (\tau_i^\top v)\, \tau_i
  \qquad \text{for any } v \in \mathbb{R}^2.
\]

Apply this to the right-hand side of~\eqref{eq:lohe} with
$\Omega_i = \omega_i J$ where
$J = \begin{pmatrix} 0 & -1 \\ 1 & 0 \end{pmatrix}$ is the $90^\circ$
rotation. The first term satisfies $\Omega_i x_i = \omega_i \tau_i$. The second
term projects the coupling onto the tangent:
\[
  (I - x_i x_i^\top) \sum_j w_{ij} x_j
  = \sum_j w_{ij}\, (\tau_i^\top x_j)\, \tau_i.
\]
Using $\tau_i^\top x_j = -\sin\theta_i \cos\theta_j + \cos\theta_i \sin\theta_j
= \sin(\theta_j - \theta_i)$,
the full Lohe equation~\eqref{eq:lohe} becomes
\[
  \dot\theta_i\, \tau_i
  = \omega_i\, \tau_i
  + \Big(\sum_j w_{ij} \sin(\theta_j - \theta_i)\Big)\, \tau_i.
\]
Both sides are scalar multiples of $\tau_i$, so the equation reduces to its
scalar coefficient:
\[
  \dot\theta_i = \omega_i + \sum_j w_{ij} \sin(\theta_j - \theta_i),
\]
which is the classical Kuramoto model~\citep{YK:75,SHS:00}. On $\mathbb{S}^1$,
the projector $(I - x_i x_i^\top)$ restricts motion to the tangent line at
$x_i$, which is one-dimensional; the dynamics therefore reduce to a single
scalar ODE for the phase angle $\theta_i$. The same reduction does not occur in
higher dimensions: for $\dosc \geq 3$ the tangent space to
$\mathbb{S}^{\dosc-1}$ at $x_i$ is $(\dosc-1)$-dimensional, and no scalar
parametrization captures the full state.


\section{Experimental details}
\label{app:details}

This appendix provides the hyperparameters, training procedures, and
infrastructure needed to reproduce the experiments. Throughout, $d_{\rm model}$
is the token embedding width, $n_h$ the number of attention heads, $n_\ell$ the
number of transformer layers, and $d_{\rm ff}$ the feed-forward width of the
transformer block.

\noindent
\textbf{Training and inference modes.} All models are trained with the analytic
fixed point $z_i^* = h_i / \norm{h_i}$ as the inference computation, with
gradients flowing through the normalization via PyTorch autograd. The
normalization is implemented with PyTorch's \texttt{F.normalize(h, dim=-1,
  eps=$10^{-8}$)}, which lower-bounds $\norm{h_i}$ by $\varepsilon = 10^{-8}$ in
the denominator to avoid the numerical singularity at $\norm{h_i}=0$; on the
measure-zero degenerate set bounded by Proposition~\ref{prop:degenerate}, the
resulting attention is approximately uniform. The analytic fixed point and the
equilibration trajectory of the ODE~\eqref{eq:fixedquery} converge to the same
point under the hypotheses of Theorem~\ref{thm:uniqueness}, so no algorithmic
gap is introduced between training and hardware inference. ODE inference at test
time (the RK45 verification reported at the end of
Section~\ref{sec:convergence}) uses \texttt{scipy.integrate.solve\_ivp} with
\texttt{method='RK45'}, $rtol=atol=10^{-4}$, and $T_{\max}=30$. The KWS
frozen-$W_V$ ablation and SVA frozen-$W_V$ ablation (Tables~\ref{tab:kws}
and~\ref{tab:sva}) use the same analytic fixed-point inference
($z_i^* = h_i/\|h_i\|$) as all other experiments.

\noindent
\textbf{Cost accounting (Tables~\ref{tab:cost}
    and~\ref{tab:cost_softplus}).} Counts cover the step where the mechanisms
  differ, per inference and per attention layer. Per-pair terms use the pairs
  actually evaluated, $P = T^2$ for the two bidirectional tasks and
  $P = T(T{+}1)/2$ for causal TinyStories. Per head: one nonlinearity per pair;
  the row sum costs one addition per term summed, $P - T$; the division one per
  weight, $P$; the readout and the fixed point $\dosc$ products per pair each;
  and the normalization of $h_i$ onto the sphere $2\dosc$ operations per
  query. An exponential is counted as ten multiply-add equivalents and every
  other operation as one.

\noindent
\textbf{Antipodal validation (Figure~\ref{fig:antipodal}).} For each
$\dosc \in \{2, 4, 8, 16, 32, 64\}$, $N=1000$ samples are drawn uniformly from
$\mathbb{S}^{\dosc-1}$ by sampling $z \sim \mathcal{N}(0, I_{\dosc})$ and
normalizing $z/\norm{z}$.  The empirical fraction within angular distance
$\alpha$ of a fixed pole is computed for $\alpha$ on a uniform grid of $50$
points in $[0.01, \pi/2]$. The closed-form prediction is computed via
\texttt{scipy.integrate.quad} on the integrand $\sin^{\dosc-2}(\theta)$ on
$[0, \alpha]$, multiplied by the prefactor
$\Gamma(\dosc/2)/(\sqrt{\pi}\,\Gamma((\dosc-1)/2))$ from \eqref{eq:capprob}.

\noindent
\textbf{KWS (Tables~\ref{tab:kws}, \ref{tab:readout}).}  $d_{\rm model}=32$,
$n_h=2$, $n_\ell=1$, $d_{\rm ff}=128$. Input is log-mel spectrograms with 40
bins, 25\,ms windows with 10\,ms hop, $T=49$ frames per utterance.
The position-matched comparison of Figure~\ref{fig:kws} (left) trains both
mechanisms under three positional-encoding conditions (none, sinusoidal,
learned absolute), five seeds per condition, identical optimizer settings.
In the no-positional-encoding condition the anchors still carry positional information, since
they are parameterized per absolute position; the frame-permutation probe of
Section~\ref{sec:bidirectional} quantifies this. The ablation rows of
Table~\ref{tab:kws} use the no-positional-encoding ablation condition. Training: AdamW optimizer with weight decay $10^{-4}$,
$lr=10^{-3}$, batch 64, 30 epochs, cosine learning-rate schedule, gradient
clipping at 1.0. The full-model, softmax-baseline, frozen-$W_V$, and
$p$-ablation use $5$~seeds.

\textit{Readout sharpening ablation (Table~\ref{tab:readout}):} same
architecture as the main KWS experiment ($d_{\rm model}{=}32$, $n_h{=}2$,
$n_\ell{=}1$), $\dosc{=}2$ throughout. Each value of $p \in \{1, 2, 4\}$ is
trained with 5 seeds. Identical training setup to the main KWS runs
except for the readout exponent.

\textit{Coupling-function ablation:} softplus is replaced by
  $\mathrm{relu}(x)+10^{-3}$ and by $\mathrm{elu}(x)+1$ at identical
  configurations. On KWS the three couplings are indistinguishable
  ($88.76$--$88.77\%$, 5 seeds each). On TinyStories (5 seeds each),
  $\mathrm{elu}(x)+1$ matches softplus ($9.782 \pm 0.058$ versus
  $9.784 \pm 0.060$), and the exponential-free $\mathrm{relu}(x)+10^{-3}$
  ($9.856 \pm 0.084$) differs by only $+0.07$ perplexity, within the
  between-seed spread (pooled std $0.10$). The
  softplus exponential is inessential.

\noindent
\textbf{SVA (Tables~\ref{tab:sva}, \ref{tab:attention}).}  Synthetic
Linzen-style sentences~\citep{TL-ED-YG:16} of the form ``\textit{The keys on the
  table are/is}'' with 40K training, 4K validation, and 4K test examples;
subject, distractor, and verb positions are recorded for attention
analysis. Sinusoidal positional encoding~\citep{AV-NS-NP-JU-LJ-AG-LK-IP:17}
added to the embeddings before the first attention layer. Training: AdamW with
weight decay $10^{-4}$, $lr=5\times10^{-4}$, batch 64, 20 epochs.

\textit{Standard architecture:}
$d_{\rm model}=64$, $n_h=2$, $n_\ell=2$, $d_{\rm ff}=256$. Both
softmax and oscillator reach $98\%$ overall (5 seeds each). At this
scale the task is fully within transformer capacity and the comparison
does not stress the attention mechanism.

\textit{Minimum-hardware configuration (Table~\ref{tab:sva}):}
$d_{\rm model}=32$, $n_h=1$, $n_\ell=1$, $d_{\rm ff}=64$. This is the smallest
configuration tested that reveals a statistically meaningful difference. The
main comparison uses 50 seeds per mechanism. Training failure counts are
compared with Fisher's exact test (two-sided) on the numbers of successful and
failed runs per mechanism, and hard-accuracy distributions with the
Mann--Whitney $U$ test. The sensitivity sweep varies learning rate
($\{2.5\!\times\!10^{-4}, 5\!\times\!10^{-4}, 10^{-3}\}$) and training epochs
($\{20, 40\}$) with 5 seeds per cell. The frozen-$W_V$ ablation uses 50
seeds. Total oscillators per attention layer:
$n_h \times T \times (\dosc-1) = 1 \times 9 \times 1 = 9$ (all hard sentences
are $T{=}9$ tokens).

\textit{Architectural robustness sweep:} To verify that the minimum-hardware
result is not fragile to the specific $d_{\rm ff}$ choice, we trained
models across a range of $d_{\rm ff}$ values (5 seeds each) while
holding all other hyperparameters fixed.  Table~\ref{tab:sva_sweep} reports hard-sentence
accuracy.  At $d_{\rm ff}=72$, the oscillator is stable ($96.76 \pm 0.46\%$) while
softmax shows bimodal outcomes ($91.07 \pm 10.03\%$), a $+5.69$ pp advantage.  At
$d_{\rm ff}\leq 56$, both models encounter bimodal training outcomes within the
20-epoch horizon, consistent with the measure-zero initialization risk near the
unstable equilibrium at severely constrained capacity.

\begin{table}[t]
\centering
\begin{tabular}{cccc}
\toprule
$d_{\rm ff}$ & Softmax hard (\%) & Oscillator hard (\%) & $\Delta$ \\
\midrule
72 & $91.07 \pm 10.03$ & $96.76 \pm 0.46$ & $+5.69$ pp \\
64 & $92.11 \pm 7.87$ & $97.38 \pm 0.33$ & $+5.27$ pp \\
56 & $83.64 \pm 9.69$ & $92.42 \pm 9.85$ & $+8.78$ pp \\
32 & $92.28 \pm 11.17$ & $94.41 \pm 4.40$ & $+2.13$ pp \\
16 & $95.32 \pm 2.14$ & $92.23 \pm 10.65$ & $-3.09$ pp \\
\bottomrule
\end{tabular}
\caption{\textbf{Does the minimum-hardware result depend on the
      feed-forward width?} We vary $d_{\rm ff}$, the feed-forward width of the
    transformer block, holding everything else at the minimum-hardware
    configuration ($d_{\rm model}=32$, $n_h=1$, $n_\ell=1$, $\dosc=2$) and
    training five seeds at each width. The two columns are hard-sentence
    accuracy for each mechanism. At $d_{\rm ff}=72$ the oscillator is stable
    while one softmax seed collapses to $73.2\%$. A collapse of this kind shows
    in the standard deviation rather than the mean: every collapsed run falls
    below $87\%$ and every other run stays above $91\%$, so a split spreads the
    seeds instead of shifting them together. At the narrower $d_{\rm ff}$ widths
    both mechanisms collapse on some seeds, so neither trains reliably once the
    block is this narrow and the comparison there says little. The oscillator's
    mean is higher at every width except the narrowest, $d_{\rm ff}=16$, where
    it is $3.09$~pp behind.}
\label{tab:sva_sweep}
\end{table}

\noindent
\textbf{WikiText-2 (Table~\ref{tab:lm}).} $d_{\rm model}=128$, $n_h=4$,
$n_\ell=2$, $d_{\rm ff}=512$, causal masking, sinusoidal positional
encoding. Word-level tokenization with vocabulary size 10K. Training: AdamW with
weight decay $10^{-4}$, $lr=5\times10^{-4}$, batch 64, 30
epochs. Analytic-fixed-point training uses 5 seeds at each of
$\dosc \in \{2, 4, 8, 16, 32\}$. The $\dosc=2$ trained model is evaluated under
three inference modes: analytic, and finite-horizon integration from random and
from sequential initialization. The random mode uses 5 initialization seeds. At
$\dosc{=}2$ the scalar Lohe ODE admits a closed-form solution, which we evaluate
at $T_{\max}=30$ rather than integrating.

\noindent
\textbf{TinyStories (Table~\ref{tab:lm}).} Same backbone as
WikiText-2 ($d_{\rm model}{=}128$, 4 heads, 2 layers), with context
length $T{=}128$ rather than $50$. Word-level tokenization with vocabulary size 8K. Training: AdamW
with weight decay $10^{-4}$, $lr=5\times10^{-4}$, batch 256, 5 epochs. All five
$\dosc \in \{2, 4, 8, 16, 32\}$ values use 5 seeds.

\textit{Readout sharpening ablation (Table~\ref{tab:readout}):} TinyStories with
same architecture as in Table~\ref{tab:lm}, varying $\dosc$ and $p$. The
$\dosc{=}8$ and $\dosc{=}2$ rows use 5 seeds.

\noindent
\textbf{ODE convergence verification (Table~\ref{tab:convergence},
  Figure~\ref{fig:convergence}).} RK45 via \texttt{scipy.solve\_ivp},
$rtol=atol=10^{-4}$, $T_{\max}=30$, evaluated for each trained TinyStories model
on 32 validation sequences of 128 tokens ($2{,}895$ token positions $\times\,4$
heads $= 11{,}580$ oscillators) with 5 random initializations of $z(0)$ per
oscillator.
Convergence is measured by the final-time error
$\mathrm{err}_i = \norm{z_{\rm ODE}(T_{\max}) - z_i^*}_2$ between the integrated
trajectory and the analytic fixed point.  The ``Converged'' column reports
trajectories with $\mathrm{err}_i < 0.01$; ``apparent-antipodal'' trajectories
have $\mathrm{err}_i > 0.1$ (slow convergence consistent with initialization
near $-z_i^*$).

\noindent
\textbf{Compute.} All experiments were run on a Mac Studio (M4 Max) using
PyTorch with the MPS backend. Antipodal validation and ODE convergence
verification run on CPU.

\textbf{Code and reproducibility.} The repository provides the model
implementation, the training and analysis code, the experiment configurations,
and the result files from which the reported tables and figures are generated.
Trained checkpoints are not distributed; analyses that rely on them can be
reproduced by~retraining.

\section{Physical robustness, settling, and attention
statistics}\label{app:robustness}

This appendix reports simulation measurements of the physical
non-idealities that a hardware realization must tolerate, together with
attention statistics used in the main text. All measurements use trained
checkpoints.

\textbf{Perturbation models.} Coupling mismatch is multiplicative: each coupling
is scaled by $e^{\epsilon}$ with $\epsilon \sim N(0, s^2)$, so at $s = 0.5$
individual couplings are off by tens of percent. State noise adds a perturbation
of scale $s$ to the oscillator states. Frequency disorder adds a skew-symmetric
generator of scale $s$ per oscillator (natural-frequency). Finite settling
truncates the dynamics at a dimensionless horizon $T$.

\textbf{Results (Figure~\ref{fig:robustness}).} Task accuracy is flat
under coupling mismatch up to $s = 0.5$ on KWS, with TinyStories perplexity
rising by about $1\%$ at the extreme. Degradation under state noise is
smooth. Frequency disorder causes no loss of locking anywhere in the
measured range $s \in [0.01, 0.5]$: KWS accuracy stays flat within
$\pm 0.1$~pp, and TinyStories perplexity rises by $3.4\%$ at the extreme.
Under finite settling, the RK45-integrated task metric converges to the
analytic fixed-point value on all three tasks (Table~\ref{tab:settle}), each
residual comparing the settled metric against the analytic fixed point on the
same evaluation subset. SVA descends monotonically to the metric's resolution
by $T{=}5$, KWS is at that resolution even at the smallest horizon tested, and
on TinyStories, the one continuous metric, the residual falls to $0.14\%$ at
$T{=}30$.
RK45 agrees with the analytic fixed point at the state and task-metric
level, and forward Euler agrees except at rare large-$\norm{h_i}$ tokens,
where its fixed step becomes unstable, so these conclusions do not depend on
the choice of integration scheme.

\begin{figure}[t]
\centering
\includegraphics[width=\linewidth]{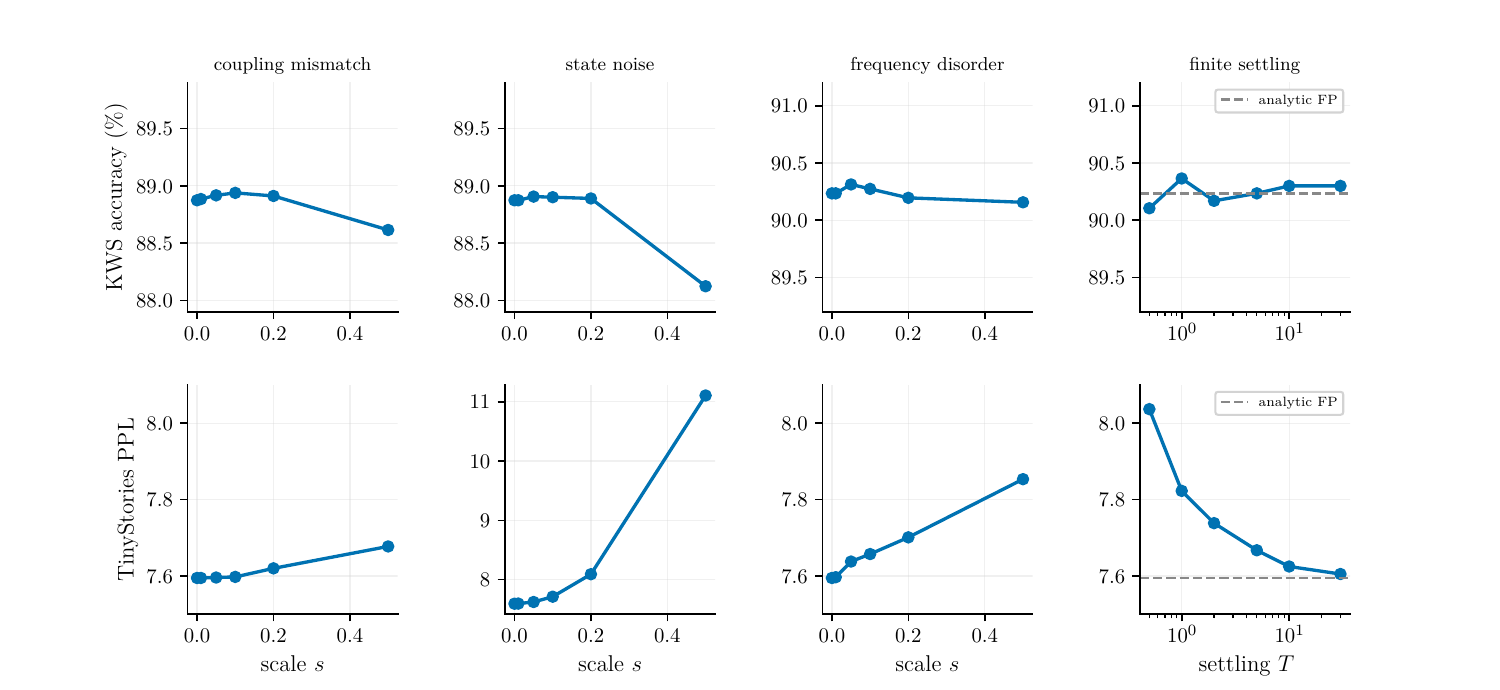}
\caption{\textbf{Robustness to physical non-idealities.} Task metric versus
perturbation scale $s$ for coupling mismatch, state noise, and frequency
disorder, and versus the settling horizon $T$, for KWS (top) and
TinyStories (bottom). Dashed lines mark the analytic fixed-point reference
on each panel's evaluation set.}
\label{fig:robustness}
\end{figure}

\begin{table}[t]
\centering
\caption{\textbf{How close the integrated dynamics come to the analytic
      fixed point.} At each horizon $T$ we integrate with RK45 and evaluate
    accuracy on KWS and SVA and perplexity on TinyStories, then report how far
    each lies from the same quantity computed with the analytic fixed point, in
    its own units: accuracy points and perplexity. Both are computed on the same
    evaluation subset. Accuracy changes in steps of one example, so a single
    example is $0.2$~pp on the
    $512$-example KWS subset and about $0.005$~pp on SVA pooled over $4000$
    sentences and five seeds; KWS entries therefore sit below what one example
    can resolve.}
\label{tab:settle}
\begin{tabular}{lccc}
\toprule
$T$ & KWS (pp) & SVA (pp) & TinyStories (PPL) \\
\midrule
$0.5$ & $-0.13$ & $-7.21$ & $+0.44$ \\
$1$   & $+0.13$ & $-2.39$ & $+0.23$ \\
$2$   & $-0.07$ & $-0.45$ & $+0.14$ \\
$5$   & $+0.00$ & $-0.00$ & $+0.07$ \\
$10$  & $+0.07$ & $-0.00$ & $+0.03$ \\
$30$  & $+0.07$ & $-0.00$ & $+0.01$ \\
\bottomrule
\end{tabular}
\end{table}

\textbf{Degenerate-set statistics.} On trained checkpoints across all three
tasks, the fraction of oscillators (token positions $\times$ heads) with
$\norm{h_i} < 0.01$ is none on SVA, $0.43\%$ (KWS; $0.025\%$ with
no-positional-encoding), and $0.11\%$ on TinyStories (at $\dosc{=}8$): the
degenerate set of Proposition~\ref{prop:degenerate} remains negligible after
training.

\textbf{Attention entropy.} Figure~\ref{fig:entropy_seeds} (left) shows the
distribution of normalized attention-row entropy on trained TinyStories
models. The oscillator readout is diffuse: median $0.90$--$0.92$ at every
$\dosc$, including $\dosc = 2$, versus $0.60$ for softmax. This measurement
supports the selectivity discussion of Section~\ref{sec:readout}.

\begin{figure}[t]
\centering
\includegraphics[width=\linewidth]{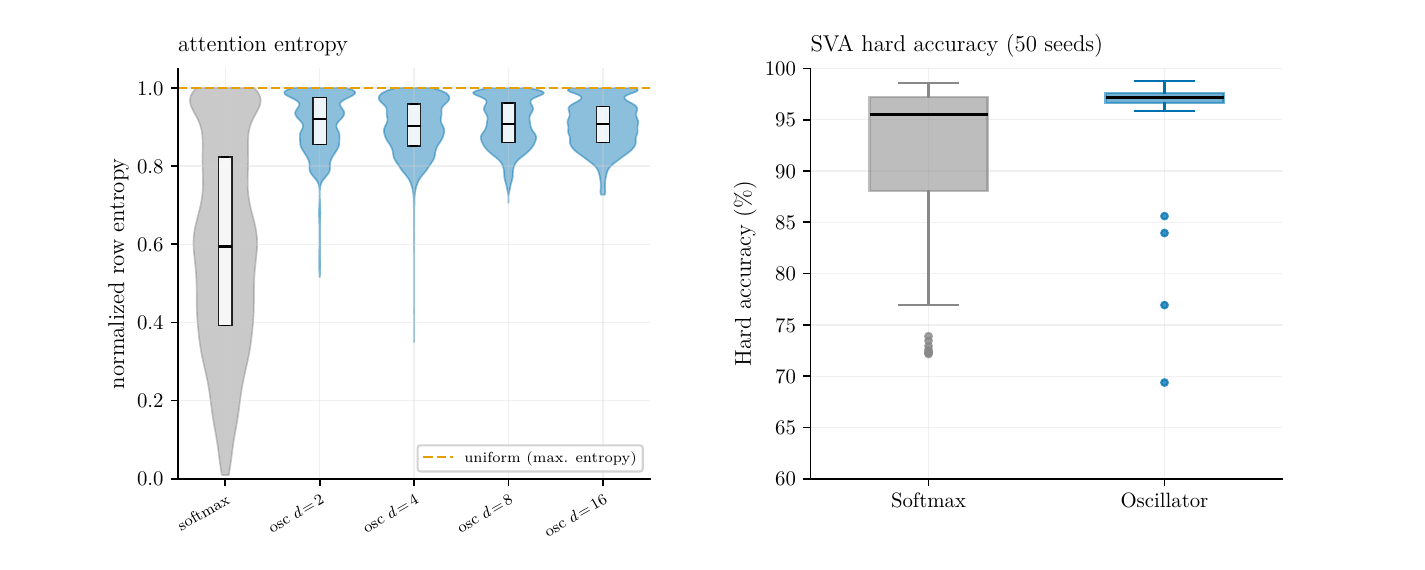}
\caption{\textbf{Attention entropy and SVA seed distribution.}
\emph{Left:} normalized attention-row entropy (uniform $= 1$) for softmax
and for the oscillator at $\dosc \in \{2, 4, 8, 16\}$; violins over raw
per-row samples with median/IQR overlays. \emph{Right:} SVA hard accuracy
across 50 seeds per mechanism; boxes show median and interquartile range,
and points beyond the whiskers are the failed seeds.}
\label{fig:entropy_seeds}
\end{figure}

\textbf{SVA seed distribution.} Figure~\ref{fig:entropy_seeds} (right)
shows the per-seed hard-accuracy distribution behind
Table~\ref{tab:sva}. Most runs of both mechanisms land in a narrow
high-accuracy band; the rest collapse partially, to between $69\%$ and $85\%$.
That happens in $3/50$ oscillator runs and $12/50$ softmax runs.

\section{Cost with softplus coupling and absolute
    counts}\label{app:cost_softplus} Table~\ref{tab:cost} assumes ReLU
  coupling. The trained models use softplus, whose exponential is counted as ten
  multiply-add equivalents like softmax's.  Table~\ref{tab:cost_softplus}
  repeats the four implementations on that assumption, and
  Table~\ref{tab:cost_absolute} gives the operation counts behind both tables.

\begin{table}[h]
\centering
\begin{tabular}{lccc}
\toprule
Reduction in attention-layer operations relative to softmax & KWS & SVA & TinyStories \\
\midrule
Softmax (reference)                            & $1.0$ & $1.0$ & $1.0$ \\
\midrule
Oscillator, all digital                        & $0.7$ & $0.7$ & $0.4$ \\
Oscillator, equilibration physical             & $0.9$ & $0.9$ & $0.6$ \\
Oscillator, equilibration and readout physical & $1.0$ & $1.0$ & $1.0$ \\
Oscillator, affine sum and division physical   & $1.2$ & $1.2$ & $1.2$ \\
\bottomrule
\end{tabular}
\caption{\textbf{Cost with softplus coupling.} Factor by which each
  implementation reduces the operation count relative to softmax. With the
  readout in the array, both mechanisms reduce to the same expression, one
  exponential per query-key pair followed by one row sum and one division, so
  the third row is exactly $1.0$. Even with the whole step physical the fourth
  row reaches only $1.2$, against $12$ with ReLU coupling in
  Table~\ref{tab:cost}.}
\label{tab:cost_softplus}
\end{table}

\begin{table}[h]
\centering
\begin{tabular}{lcccccc}
\toprule
& \multicolumn{3}{c}{ReLU coupling} & \multicolumn{3}{c}{softplus coupling} \\
\cmidrule(lr){2-4}\cmidrule(lr){5-7}
Operations in the attention layer & KWS & SVA & TS & KWS & SVA & TS \\
\midrule
Softmax                                        & $57{,}526$ & $963$ & $395{,}776$ & $57{,}526$ & $963$ & $395{,}776$ \\
\midrule
Oscillator, all digital                        & $33{,}908$ & $594$ & $635{,}136$ & $77{,}126$ & $1{,}323$ & $932{,}352$ \\
Oscillator, equilibration physical             & $23{,}912$ & $396$ & $362{,}752$ & $67{,}130$ & $1{,}125$ & $659{,}968$ \\
Oscillator, equilibration and readout physical & $14{,}308$ & $234$ & $98{,}560$  & $57{,}526$ & $963$ & $395{,}776$ \\
Oscillator, affine sum and division physical   & $4{,}802$  & $81$  & $33{,}024$  & $48{,}020$ & $810$ & $330{,}240$ \\
\bottomrule
\end{tabular}
\caption{\textbf{Absolute operation counts behind Tables~\ref{tab:cost}
    and~\ref{tab:cost_softplus},} in multiply-add equivalents, per inference and
  per attention layer. The counting convention is stated in
  Appendix~\ref{app:details}.}
\label{tab:cost_absolute}
\end{table}

\section{Measuring the dimensional bottleneck in causal language
  modeling}\label{app:alternative_ablations}
This appendix measures attention rank directly and tests the
alternative explanations.

\textbf{Measured effective rank.} Throughout we use the entropy-based
effective rank of \citet{roy2007effective}, computed as the exponential of
the entropy of the normalized singular values. It equals $r$ for a matrix
with $r$ equal singular values and is low when a few directions dominate;
it can never exceed the exact rank. Stack the settled states
$z_i^{*\top}$ into $Z^* \in \mathbb{R}^{T \times \dosc}$ and the anchors
$r_j^\top$ into $R \in \mathbb{R}^{T \times \dosc}$. The similarity matrix
before masking is $S = \mathbf{1}\mathbf{1}^\top + Z^* R^\top$, the sum of
a rank-one shift and a product of two rank-$\dosc$ factors, so
$\mathrm{rank}(S) \le \dosc + 1$. On trained models the measured effective
rank of $S$ is $2.31$, $3.95$, $5.79$, and $9.59$ for
$\dosc = 2, 4, 8, 16$ (Figure~\ref{fig:rank}, left): the trained models operate
close to, but below, their ceiling at every $\dosc$. Masking changes the
picture: it is an entrywise operation and does not preserve low rank, and the
causal mask is
invertible, so the bound no longer applies after it. The realized attention
matrices, after masking and row renormalization, have measured effective
rank around $40$.

\textbf{Rank as a resource.} Figure~\ref{fig:rank} (right) asks how costly a
rank restriction is, comparing three series evaluated on the same full
validation set. Truncating the trained softmax model's attention to rank $r$ at
inference degrades steeply as the rank is reduced, from $10.6$ at $r = 33$ to
$49.7$ at $r = 3$ (higher perplexity is worse), and stays above both trained
mechanisms at every rank. Training softmax under the restriction instead, with
the per-head query/key dimension reduced so that its score rank matches the
oscillator's ceiling, gives a clean monotone descent over
$d_h \in \{2, 4, 8, 16, 32\}$, each point a mean over five seeds:
$9.41 \to 9.18 \to 8.96 \to 8.83 \to 8.67$, the $d_h{=}32$ point being the
unmodified full-model softmax baseline. Oscillator attention trained under its
own ceiling degrades gently over the same range, from $10.91$ at $\dosc = 2$ to
$9.11$ at $\dosc = 32$. The comparison between the first series and the other
two is deliberately asymmetric: rank restriction is very costly for a model not
adapted to it, and training under the constraint recovers most of that cost.
Restricting softmax's score rank to match the oscillator's costs $0.74$ PPL
($9.41$ against the $8.67$ baseline), while the oscillator's gap at $\dosc{=}2$
is $2.24$ PPL ($10.91$ against $8.67$), so score rank accounts only for part of
the gap.

\begin{figure}[t]
\centering
\includegraphics[width=\linewidth]{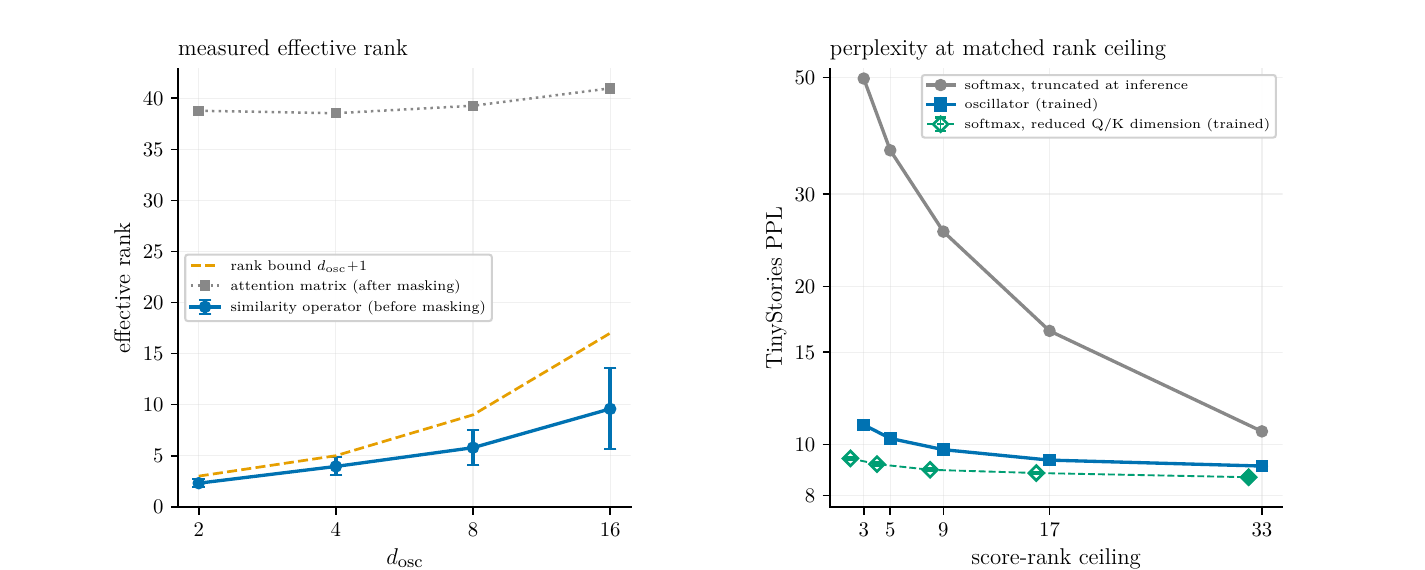}
\caption{\textbf{Measured attention rank and perplexity at matched rank
ceiling.} \emph{Left:} effective rank of the similarity operator before masking
(whiskers: $\pm 1$ std over per-sequence, per-head, per-layer samples),
against the $\dosc{+}1$ bound, and of the realized attention matrices after
masking. \emph{Right:} TinyStories perplexity at matched rank ceiling for softmax
truncated at inference, softmax trained with reduced query/key dimension,
and the oscillator trained under its ceiling; each series is plotted at its
own ceiling.}
\label{fig:rank}
\end{figure}

\textbf{Elimination of alternatives.} We tested three modifications that
  could plausibly close the oscillator--softmax perplexity gap without
  increasing $\dosc$: learned position-dependent anchor phase offsets, learned
  per-head per-position coupling amplification, and scaling the number of heads
  and $d_{\rm model}$. All use the WikiText-2 architecture
  ($d_{\rm model}{=}128$, $n_h{=}4$, $n_\ell{=}2$, $\dosc{=}2$) and the same
  optimizer and training budget as the headline runs, with mean $\pm$ std over 5
  seeds. Oscillator attention has no explicit notion of
  token distance, since the anchor $r_j$ depends only on the content of token
  $j$ and not on the offset $i{-}j$, so we added a learned rotation
  $\delta(i{-}j)$ to each anchor before computing the fixed point; this changed
  PPL by less than half a point ($-0.42\pm0.25$ against the $\dosc{=}2$
  baseline), because position already enters through the input embeddings and
  routing it through the anchor coupling adds little further
  capacity. Amplifying the couplings by a learned scalar $\beta_{h,i}$ per
  (head, query-position) pair left PPL negligibly affected ($-0.06\pm0.44$),
  even though the trained $\beta_{h,i}$ span $[0.10, 2.68]$ and so do
  vary by position: a scalar gain cannot enlarge the $\dosc$-dimensional anchor
  subspace that governs attention rank, which locates the bottleneck in the
  anchor dimension rather than in the gain. Scaling capacity moved both
  mechanisms without closing the distance between them. Starting from the
  five-seed WikiText-2 $\dosc{=}2$ configuration of Table~\ref{tab:lm}
  ($110.67 \pm 0.51$ oscillator, $99.00 \pm 0.48$ softmax), raising $n_h$ from
  $4$ to $8$ at fixed $d_{\rm model}{=}128$ degraded the oscillator to
  $111.13 \pm 0.72$ and left softmax essentially unchanged at $98.81 \pm 0.35$,
  consistent with the smaller per-head dimensionality ($d_h{=}16$ instead of
  $32$), while raising both together ($8$ heads at $d_{\rm model}{=}256$)
  improved the oscillator to $104.42 \pm 0.49$ and softmax to $92.49 \pm 0.46$,
  leaving the gap at $11.93$ PPL against $11.67$ at baseline.  Each result is
  consistent with the dimensional-bottleneck diagnosis of Section~\ref{sec:lm},
  and none closes the gap independently of $\dosc$.


\section{Code availability}\label{app:code}
Supporting code is available at
\mbox{\url{https://github.com/FabioPasqualetti/attention-by-synchronization}}.

\end{document}